\documentclass{article}

\usepackage{subfigure, epsfig, graphics}
\usepackage{amsmath,amssymb,amsthm}
\usepackage[numbers]{natbib}
\usepackage{algorithm,algorithmic}
\usepackage{epstopdf}

\newtheorem{lemma}{Lemma}
\newtheorem{theorem}{Theorem}

\newcommand\tr[2]{\left\langle#1,#2\right\rangle}

\def\svert{\ensuremath{k}}

\def\real{\mathbb{R}}

\newcommand\trace[1]{\text{trace}\left(#1\right)}
\def\L{\mathcal{L}}
\def\supp{\text{Supp}}

\def\EstDual{\widetilde{Z}}
\def\sgn{\text{Sign}}
\def\T{\mathcal{T}}
\def\P{\mathcal{P}}
\def\O{\Omega}
\def\tA{\widetilde{A}}
\def\tL{\widetilde{L}}
\def\hA{\widehat{A}}
\def\hL{\widehat{L}}
\def\S{\mathcal{S}}

\def\H{\mathcal{H}}
\def\G{\mathcal{G}}
\def\tU{\widetilde{U}}
\def\tV{\widetilde{V}}
\def\tSigma{\widetilde{\Sigma}}
\def\tT{\widetilde{\mathcal{\T}}}
\def\defn{:=}
\def\N{\mathcal{N}}

\title{Learning the Dependence Graph of Time Series with Latent Factors}

\author{Ali Jalali \\ University of Texas at Austin\\ alij@utexas.edu \and Sujay Sanghavi\\ University of Texas at Austin\\ sanghavi@mail.utexas.edu}

\begin{document}
\maketitle

\begin{abstract}
This paper considers the problem of learning, from samples, the dependency structure of a system of linear stochastic differential equations, when some of the variables are {\em latent}. In particular, we observe the time evolution of some variables, and never observe other variables; from this, we would like to find the dependency structure between the observed variables -- {\em separating out} the spurious interactions caused by the (marginalizing out of the) latent variables' time series. We develop a new method, based on convex optimization, to do so in the case when the number of latent variables is smaller than the number of observed ones. For the case when the dependency structure between the observed variables is sparse, we theoretically establish a high-dimensional scaling result for structure recovery. We verify our theoretical result with both synthetic and real data (from the stock market).
\end{abstract}

\section{Introduction}

Linear stochastic dynamical systems are classic processes that are widely used to model time series data in a huge number of fields: financial data \cite{COC}, biological networks of species \cite{LAWGIRRATSAN} or genes \cite{BAR}, chemical reactions \cite{GIL,HIG}, control systems with noise \cite{YOU}, etc. An important task in several of these domains is learning the model from data \cite{VAP}; doing so is often the first step in both data interpretation, and making predictions of future values or the effect of perturbations. Often one is interested in learning the {\em dependency structure} \cite{JOR}; i.e. identifying, for each variable, which set of other variables it directly interacts with. For stock market data, for example, this can reveal which other stocks most directly affect a given stock. 

We consider model structure learning in a particularly challenging yet widely prevalent setting: where (the time series of) some state variables are observed, and others are {\em unobserved/latent}. We are interested in learning the dependency structure between the observed variables. However, the presence of latent time series, if not properly accounted for by the model learning procedure, will result in the appearance of spurious interactions between observed variables -- two observed variables that interact with the same unobserved variable may now be reported to be interacting. This happens, for example, if one uses the classic maximum-likelihood estimator \cite{FIS}, and persists even if we have observations over a long time horizon. 

Suppose, for illustration, that we are interested  in learning the dependency structure between the prices of a set of stocks via a linear stochastic model. Clearly, stock prices depend not only on each other, but are also jointly influenced by several variables that may not be part of our model, for example, currency markets, commodity prices etc.; these are latent time series. Their presence means that a naive structure learning algorithm (say max-likelihood) that takes as input only the stock prices, will report several spurious interactions; say, e.g. between all stocks that fluctuate with the price of oil.

Our work involves several significant differences from the large body of work on sparse recovery and graphical model learning. One is the fact that our samples are dependent on each other, with the degree of dependence governed by how finely the system is sampled. Another is the presence of latent variables. We put our work in context with related literature in Section \ref{sec:related}.

Clearly there are several issues with regards to fundamental identifiability, and sample and computational complexity, that need to be defined and resolved. We do so below in the specific context of our model setting. We provide both theoretical characterization and guarantees of the problem, as well as numerical illustrations for both synthetic data and some real data extracted from stock market.  

The rest of the paper is organized as follows: We review the related literature in Section~\ref{sec:related}. We present the main idea and our algorithm in Section~\ref{sec:setting}. Section~\ref{sec:main} reveals our main result followed by the proof in Section~\ref{sec:proof}. The simulation results are included in Section~\ref{sec:exper}.

\section{Related Work} \label{sec:related}

We organize the most directly related work as follows  (recognizing of course that these descriptions overlap).\\

{\bf Sparse Recovery and Gaussian Graphical Model Selection:} It is now well recognized \cite{TIB,WAI,MEIBUL} that a sparse vector can be tractably recovered from a small number of linear measurements; and also that these techniques can be applied to do model selection (i.e. inferring the Markov graph structure and parameters) in Gaussian graphical models \cite{MEIBUL,RAVWAIRASYU,Asp07b,glasso,YuaLin07}. While ours problem is, in a sense, also one of sparse linear model selection, two differences between our setting and these papers is that they do not have any latent factors, and theoretical guarantees typically require independent samples. The simultaneous presence of both these characteristics is what makes ours a challenging setting. In particular, latent factors imply that these techniques will in effect attempt to find models that are dense, and hence not be able to have a high-dimensional scaling. Correlation among samples means we cannot directly use standard concentration results, and also brings in the interesting issue of the effect of sampling frequency; in particular, in our setting one can get more samples by finer sampling, but increased correlation means these do not result in better consistency. \\

{\bf Sparse plus Low-Rank Matrix Decomposition:} Our results are based on the possibility of separating a low-rank matrix from a sparse one, given their sum (either the entire matrix, or randomly sub-sampled elements thereof) -- see \cite{CHASANPARWIL,CANLIMAWRI,CHEJALSANCAR,ZHOLIWRICANMA,CANPLA} for some recent results, as well as its applications in graph clustering \cite{JalCheSanXu11,JalSre12}, collaborative filtering \cite{SreJaa03}, image coding \cite{HazPolsha05}, etc. Our setting is different because we observe correlated linear functions of the sum matrix, and furthermore these linear functions are generated by the stochastic linear dynamical system described by the matrix itself. Another difference is that several of these papers focus on recovery of the low-rank component, while we focus on the sparse one. These two objectives have a very different high-dimensional scaling in our linear observation setting. \\

{\bf Inference with Latent Factors:} In real applications of data driven inference, it is always a concern that whether or not there exist influential factors that have never been observed \cite{Loe04,West03bayesianfactor}. Several approaches to this problem are based on Expectation Maximization (EM) \cite{DemLaiRub77,RenWal84}; while this provides a natural and potentially general method, it suffers from the fact that it can get stuck in local optima (and hence is sensitive to initialization), and that it comes with weak theoretical guarantees. The paper  \cite{CHAPARWIL} takes an alternative, convex optimization approach to the latent factor problem in Gaussian graphical models, and is of direct relevance to our paper. In \cite{CHAPARWIL}, the objective is to find the number of latent factors in a Gaussian graphical model, given iid samples from the distribution of observed variables; they also use sparse and low-rank matrix decomposition. Differences between our paper and theirs is that we focus on recovering the support of the ``sparse part", i.e. the interactions between the observed variables exactly, while they focus on recovery the rank of the low-rank part (i.e. the number of latent variables). Our objective requires $O(\log p)$ samples, theirs requires $\Omega(p)$. Another major difference is that our observations are correlated, and hence sample complexity itself needs a different definition (viz. it is no more the number of samples, but rather the overall time horizon over which the linear system is observed).\\ 

{\bf System Identification:} Linear dynamical system identification is a central problem in Control Theory \cite{Lju99}. There is a long line of work on this problem in that field including recent regularized convex optimization based approaches \cite{FazPon11}. Recently, \cite{BENIBRMON} considered the system identification problem as learning dependence graph of time series, {\em without} any latent variables. They implement the LASSO; the main contribution is characterizing sample complexity in the presence of sample dependence. In our setting, with latent variables, their method returns several spurious graph edges caused by marginalization of latent variables.\\

{\bf Time-series Forecasting:} Motivated by finance applications, time-series forecasting has got a lot of attention during the past three decades \cite{CHA00}. In the model based approaches, it is assumed that the time-series evolves according to some statistical model such as linear regression model \cite{BOWOCO93}, transfer function model \cite{BoxJenRei90}, vector autoregressive model \cite{Wei94}, etc. In each case, researchers have developed different methods to learn the parameters of the model for the purpose of forecasting. In this paper, we focus on linear stochastic dynamical systems that are an instance of vector autoregressive models. Previous work toward estimating this model parameters include ad-hoc use of neural network \cite{Azo94} or support vector machine method \cite{Kim03}, all without providing theoretical guarantees on the performance of the algorithm. Our work is different from these results because although our method provides better prediction comparing to similar algorithm, our main focus is sparse model selection not prediction. Perhaps, once a sparse model is selected, one can study the prediction quality as a separate subject.

\section{Problem Setting and Main Idea} 
\label{sec:setting}

This paper considers the problem of structure learning in linear stochastic dynamical systems, in a setting where only a subset of the time series are observed, and others are unobserved/latent. In particular, we consider a system with state vectors $x(t)\in\real^p$ and $u(t)\in \real^r$, for $t\in\real^+$  and dynamics described by
\begin{equation}
\frac{d}{dt}\left[\begin{array}{c} x(t) \\ u(t) \end{array}\right] = \underbrace{\left[\begin{array}{cc} A^* & B^* \\ C^* & D^* \end{array}\right]}_{\mathcal{A}^*} \left[\begin{array}{c} x(t) \\ u(t) \end{array}\right] + \frac{d}{dt}w(t),
\label{cont_system}
\end{equation}
where, $w(t)\in\real^{p+r}$ is an independent standard Brownian motion vector and $A^*,B^*,C^*,D^*$ are system parameters. 

{\bf Task:} We observe the process $x(t)$ for some time horizon $0\leq t \leq T$, but not the process $u(\cdot)$. We are interested in learning the matrix $A^*$, which captures the interactions between the observed variables. 

We will also be interested in a similar objective for an analogous {\em discrete time} system with parameter $0<\eta<\frac{2}{\sigma_{\max}(\mathcal{A}^*)}\,$:
\begin{equation}
\left[\begin{array}{c} x(n+1) \\ u(n+1) \end{array}\right] - \left[\begin{array}{c} x(n) \\ u(n) \end{array}\right] = \eta\left[\begin{array}{cc} A^* & B^* \\ C^* & D^* \end{array}\right] \left[\begin{array}{c} x(n) \\ u(n) \end{array}\right] + w(n)
\label{disc_system}
\end{equation}
for all $n\in\mathbb{N}_0$. Here, $w(n)$ is a zero-mean Gaussian noise vector with covariance matrix $\eta I_{(p+r)\times (p+r)}$. The prameter $\eta$ can be thought of as the sampling step; in particular notice that as $\eta\rightarrow 0$, we recover model \eqref{cont_system} from model \eqref{disc_system}. The upper bound on $\eta$ ensures the stability of the discrete time system as required by our theorem. Intuitively, $\sigma_{\max}(\mathcal{A}^*)$ corresponds to the fastest convergence rate in the system and the upper bound on $\eta$ corresponds to the Nyquist minimum sampling rate required for the reconstruction of the signal. As done in \cite{BENIBRMON}, our proofs will initially focus on the discrete case \eqref{disc_system}, and derive results for \eqref{cont_system} afterwards. 

{\bf (A1) Stable Overall System}: We only consider stable systems. In fact, we impose an assumption slightly stronger than the stability on the overall system. For the continuous system \eqref{cont_system}, we require $D \defn -\lambda_{\max}(\frac{\mathcal{A}^* + \mathcal{A}^{*T}}{2}) > 0$. With slightly abuse of notation, for the discrete system \eqref{disc_system}, we require $D\defn\frac{1-\Sigma_{\max}^2}{\eta} > 0$, where, $\Sigma_{\max}\defn\sigma_{\max}(I+\eta \mathcal{A}^*)$. ~$\blacksquare$

As a consequence of this assumption, by Lyapunov theory, the continuous system \eqref{cont_system} has a unique stationary measure which is a zero-mean Gaussian distribution with positive definite (otherwise, it is not unique) covariance matrix $\mathcal{Q}^*\in\real^{(p+r)\times (p+r)}$ given by the solution of $\mathcal{A}^*\mathcal{Q}^* + \mathcal{Q}^*\mathcal{A}^{*T} + I = 0$. Similarly, for the discrete time system \eqref{disc_system}, we have
$\mathcal{A}^*\mathcal{Q}^* + \mathcal{Q}^*\mathcal{A}^{*T} +  \eta \mathcal{A}^*\mathcal{Q}^*\mathcal{A}^{*T} + I = 0$. This matrix $\mathcal{Q}^*$ has the form $\mathcal{Q}^* = [ Q^* \, R^{*T} \, ;\,R^* \, P^* ]$, where, $Q^*$ and $P^*$ are the steady-state covariance matrices of the observed and latent variables, respectively, and $R^*$ is the steady-state cross-covariance between observed and latent variables. By stability, $\mathcal{C}_{\min}\defn\Lambda_{\min}(\mathcal{Q}^*)>0$ and $\mathcal{D}_{\max}\defn\Lambda_{\max}(\mathcal{Q}^*)<\infty$.

{\bf Identifiability:} Clearly, the above objective of identifying $A^*$ is in general impossible without some additional assumptions on the model; in particular, several different choices of the overall model (including different choices of $A^*$) can result in the same {\em effective} model for the $x(\cdot)$ process. $x(\cdot)$ would then be statistically identical under both models, and correct identification would not be possible even over an infinite time horizon. Additionally, it would in general be impossible to achieve identification if the number of latent variables is comparable to or exceeds the number of observed variables. Thus, to make the problem well-defined, we need to restrict (via appropriate assumptions) the set of models of interest. 

\subsection{Main Idea} 
Consider the discrete-time system \eqref{disc_system} in steady state and suppose, for a moment, that we ignored the fact that there may be latent time series; in this case, we would be back in the classical setting, for which the (population version of)  the likelihood is
\begin{equation}\L(A) = \frac{1}{2\eta^2} \mathbb{E} \left[ \left\|x(i+1)-x(i)-\eta Ax(i)\right\|_2^2\right].\nonumber\end{equation}
\begin{lemma}
For $x(\cdot)$ generated by \eqref{disc_system}, the the optimum $\widehat{A} := \max_A \L(A)$ is given by \begin{equation}\widehat{A} ~ = ~ A^* + B^*R^*(Q^*)^{-1}.\nonumber\end{equation}
\label{lem:idea}
\end{lemma}
Thus, the optimal $\widehat{A}$ is a sum of the original $A^*$ (which we want to recover) and the matrix $B^*R^*(Q^*)^{-1}$ that captures the spurious interactions obtained due to the latent time series. Notice that the matrix $B^*R^*(Q^*)^{-1}$ has the rank at most equal to number $r$ of latent time series. We will assume that the number of latent time series is smaller than the number of observed ones -- i.e. $r<p$ -- and hence $B^*R^*(Q^*)^{-1}$ is a {\em low-rank matrix}.

\subsection{Identifiability}
Besides identifying the effect of the latent time series, we would need the true model to be such that $A^*$ is uniquely identifiable from $B^*R^*(Q^*)^{-1}$. We choose to study models that have a {\em local-global structure} where {\em (a)} each of the observed time series $x_i(t)$ interacts with only a few other observed series, while {\em (b)} each of the latent series interacts with a (relatively) large number of observed series. In the stock market example, for instance, this would model the case where the latent series corresponds to macro-economic factors, like currencies or the price of oil, that affect a lot of stock prices. 

In particular, let $s$ be the maximum number of non-zero entries in any row or column of $A^*$ ; it is the maximum number of other observed variables any given observed variable directly interacts with. Note that this means $A^*$ is a {\em sparse} matrix. Let $L^* := B^*R^*(Q^*)^{-1}$ and assume it has SVD $L^* = U^*\Sigma^*V^{*T}$, and recall that its rank is $r$. Then, following \cite{CHEJALSANCAR}, $L^*$ is said to be $\mu${\em-incoherent} if $\mu>0$ is the smallest real number satisfying
\begin{equation}
\begin{aligned}
&\max_{i,j}(\|U^{*T}\mathbf{e}_{i}\|, \|V^{*T}\mathbf{e}_{j}\|) \leq\sqrt{\frac{\mu r}{p}}\;\;,\;\;\|U^*V^{*T}\|_{\infty} \leq\sqrt{\frac{r\mu}{p^2}},
\end{aligned}
\nonumber
\end{equation}
where, $\mathbf{e}_{i}$'s are standard basis vectors and $\|\cdot\|$ is vector 2-norm. Smaller values of $\mu$ mean the row/column spaces make larger angles with the standard bases, and hence the resulting matrix is more dense.

{\bf (A2) Identifiability}:  We require that the $s$ of the sparse matrix $A^*$ and the $\mu$ of the low-rank $L^*$, which has rank $r$, satisfy $\alpha \defn 3 \sqrt{\frac{\mu r}{p}}<1$. ~ $\blacksquare$\\

Note that here we provide deterministic worst-case conditions on the sparse matrix. As shown in \cite{CHEJALSANCAR,CANLIMAWRI}, better scaling is possible if we pursue probabilistic guarantees.

\subsection{Algorithm}
Recall that our task is to recover the matrix $A^*$ given observations of the $x(\cdot)$ process. We saw that the max-likelihood estimate (in the population case) was the sum of $A^*$ and a low-rank matrix; we subsequently assumed that $A^*$ is sparse. It is natural to use the max-likelihood as the loss function for the {\em sum} of a sparse and low-rank matrix, and separate appropriate regularizers for each of the components. Thus, for the continuous-time system observed up to time $T$, we propose solving
\begin{equation}
\begin{aligned}
(\hA,\hL)\!\! &=\!\! \arg \min_{A,L}\,\frac{1}{2T}\!\!\int_{t=0}^{T}\!\!\!\left\|(A+L)x(t)\right\|_2^2\,dt -\!\frac{1}{T}\!\!\int_{t=0}^{T}\!\!\!x(t)^T(A+L)^Tdx(t)\\ &\qquad\qquad\qquad\qquad\qquad\qquad\qquad\qquad\qquad\qquad+ \lambda_A\|A\|_1 +  \lambda_L\|L\|_*,
\end{aligned}
\label{eq:opt-orig-cont}
\end{equation}
and for the discrete-time system given $n$ samples, we propose solving
\small\begin{equation}
\begin{aligned}
(\hA,\hL)\! =\! \arg\! \min_{A,L}\,&\frac{1}{2\eta^2n}\sum_{i=0}^{n-1}\left\|x(i+1)\!-\!x(i)\!-\!\eta(A+L)x(i)\!\right\|_2^2 + \, \lambda_A\|A\|_1\, + \, \lambda_L\|L\|_*.
\end{aligned}
\label{eq:opt-orig}
\end{equation}\normalsize
Here $\|\cdot\|_1$ is the $\ell_1$ norm (a convex surrogate for sparsity), and $\|\cdot\|_*$ is the nuclear norm (i.e. sum of singluar values, a convex surrogate for low-rank). The optimum $\hA$ of \eqref{eq:opt-orig} or \eqref{eq:opt-orig-cont} is our estimate of $A^*$, and our main result provides conditions under which we recover the support of $A^*$, as well as a bound on the error in values $\|\hA - A^*\|_\infty$ (maximum absolute value). We provide a bound on the error $\|\hL-L^*\|_2$ (spectral norm) for the low-rank part. Notice that the discrete objective function goes to the continuous one as $\eta\rightarrow 0$. 

\subsection{High-dimensional setting} 
Note that when $A^*$ is a sparse matrix, the actual degrees of freedom between the observed variables is smaller than that evinced by the ambient dimension $p$. Indeed, we will be interested in recovering $A^*$ with a number of samples $n$ that is potentially much smaller than $p$ (for small $s$). In the special case when we are in steady state and $L= 0$ (i.e. $\lambda_L$ large) the recovery of each row of $A^*$ is akin to a LASSO \cite{TIB} problem (of sparse vector recovery from noisy linear measurements) with $Q^*$ being the covariance of the design matrix. We thus require $Q^*$ to satisfy incoherence conditions that are akin to those in LASSO (see e.g. \cite{WAI} for the necessity of such conditions).

{\bf (A3) Incoherence}: To control the effect of the \emph{irrelevant} (not latent) variables on the set of \emph{relevant} variables, we require $$\theta \defn 1 - \max_\svert\|Q_{\S_\svert^c\S_\svert}^*\left(Q_{\S_\svert\S_\svert}^*\right)^{-1}\|_{\infty,1}\!>0,$$ where, $\S_\svert$ is the support of the $\svert^{th}$ row of $A^*$ and $\S_\svert^c$ is the complement of that. The norm $\|\cdot\|_{\infty,1}$ is the maximum of the $\ell_1$-norm of the rows. ~$\blacksquare$

\section{Main Results}
\label{sec:main}
In this section, we present our main result for both Continuous and Discrete time systems. We start by imposing some assumptions on the regularizers and the sample complexity.

{\bf (A4) Regularizers}: We need to impose some assumptions on the regularizers to be able to guarantee our result. Let 
\begin{equation}m\!=\!\max\left(\frac{80}{\sqrt{D}}\|B^*\|_{\infty,1}, \sqrt{\|x(0)\|_2^2+\!\!\|u(0)\|_2^2+\!\!(\sqrt{\eta}+1)^2}\,\right),\nonumber\end{equation} be the constant capturing the effect of initial condition and latent variables through matrix $B^*$. We impose the following assumptions on the regularizers:

\noindent {\bf (A4-1)} $\lambda_A=\frac{16m(4-\theta)}{\theta\sqrt{D}} \sqrt{\frac{\log\left(\frac{4((s+2r)p+r^2)}{\delta}\right)}{n\eta}}$.

\noindent {\bf (A4-2)}
$\frac{\lambda_L}{\lambda_A\sqrt{p}}=\frac{1}{1-\alpha} \left(\!\left(\frac{3\alpha\sqrt{s}}{4} \!+\!\frac{(8-\theta)s}{\theta(4-\theta)}\right)\!\!\left(\frac{\theta\sqrt{p}}{9s\sqrt{s}} \!+\!1\right)\!\!+\!\frac{1}{2}\right)$.\\

{\bf Note:} In practice, we let $\lambda_A = c \sqrt{\log\left(4((s+2r)p+r^2)/\delta\right)/n\eta}$ and $\lambda_L = d \sqrt{p} \lambda_A$, with the constants $c,d$ chosen by cross-validation over prediction performance.\\

{\bf (A5) Sample Complexity}: In our setting, samples are dependent; in particular, the smaller the $\eta$ the more dependent two subsequent samples. Sample complexity is thus governed by the total time horizon $\eta n = T$ over which we observe the system, and not simply $n$; indeed finer sampling (i.e. smaller $\eta$) requires a larger number of samples. For a probability of failure $\delta$, we require
\begin{equation}T=n\eta\geq\frac{K ~ s^3}{D^2\theta^2\mathcal{C}_{\min}^2}\log\left(\frac{4((s+2r)p+r^2)}{\delta}\right).\nonumber\end{equation} Here, $K$ is a constant independent of any other system parameter; for example, $K\geq 3 \times 10^6$ suffices.

The above $T$ is required to ensure that the empirical covariance matrix is close to the steady-state $Q^*, R^*$. Of course the constraint $\eta < 2/\sigma_{max}(\mathcal{A}^*)$ ensures that the sampling intervals cannot be too large. Note that $T$ is the total time over which the system is observed; a finer sampling cannot yield a smaller horizon, because of increased dependence between samples.

Let \small$\nu\defn\frac{\alpha\theta}{2\mathcal{D}_{\max}} +\frac{(8-\theta)\sqrt{s}}{\mathcal{C}_{\min}(4-\theta)}\,$\normalsize and \small$\rho_0\defn\min\left(\frac{\alpha}{4}, \frac{\theta\alpha\lambda_A}{5\theta\alpha\lambda_A+16\mathcal{D}_{\max}\left\|L^*\right\|_2}\right)$\normalsize.
The following (unified) theorem states our main result for both discrete and continuous time systems.\\

\begin{theorem}
If assumptions (A1)-(A5) are satisfied, then with probability $1-\delta$, our algorithm outputs a pair $(\hA,\hL)$ satisfying

\noindent {\bf (a) Subset Support Recovery: } $\supp(\hA) \subset \supp(A^*).$ 

\noindent {\bf (b) Error Bounds:} $$\|\hA-A^*\|_\infty\leq\nu\lambda_A \quad\text{and}\quad \|\hL-L^*\|_2\leq\frac{\rho_0}{1-5\rho_0}\|L^*\|_2.$$

\noindent {\bf (c) Exact Signed Support Recovery:} If additionally we have that the smallest magnitude $A_{min}$ of a non-zero element of $A^*$ satisfies $A_{min} > \nu \lambda_A$, then we obtain full signed-support recovery $\sgn(\hA)=\sgn(A^*)$.\\
\label{thr:main}
\end{theorem}

{\bf Note:} Note that $\lambda_A$, as defined in {\bf (A4-1)}, depends on the sample complexity $T$, and goes to $0$ as $T$ becomes large. Thus it is possible to get exact signed support recovery by making $T$ large.

{\bf Remark 1:} Our result shows that, in sparse  and low-rank decomposition for latent variable modeling, recovery of only the sparse component seems to be possible with much fewer samples -- $O(s^3 \log p)$ -- as compared to, for example, the recovery of the exact rank of the low-rank part; the latter was show to require $\Theta(p)$ samples in \cite{CHAPARWIL}. 

%

{\bf Remark 2:} The above theorem shows that, even in the presence of latent variables, our algorithm requires a similar number of samples (i.e. upto universal constants) as previous work \cite{BENIBRMON} required in the absence of hidden variables. Of course, this is true as long as identifiability {\bf (A2)} holds. Note that the absence of such identifiability conditions makes even simple sparse and low-rank matrix decomposition \cite{CHASANPARWIL} ill-posed. Note also that the quantity $\rho_0$, which characterizes the error in the low-rank term, goes to 0 as $T$ increases (which decreases $\lambda_A$). 

{\bf Remark 3:} Although our theoretical result shows a scaling proportional $s^3$ for the sample complexity, the theoretical result suggests that the correct scaling factor is $s^2$. We suspect our result as well as \citet{BENIBRMON}, can be tightened and we are currently working on that. 

{\bf Illustrative Example:} Consider a simple idealized example that helps give intuition about the above theorem. Suppose that we are in the continuous time setting, where each latent variable $j$ depends only on its own past, updating according to 
$\frac{dx_j}{dt} = - x_j(t) + \frac{dw_j}{dt}$ 
and for each observed variable $i$ depends only on its own past and a {\em unique} latent variable $j(i)$, i.e.,  
$\frac{dx_i}{dt} = - x_i(t) + x_{j(i)}(t) + \frac{dw_i}{dt}$.
There are $r$ latent variables, and assume that each latent variable affects exactly $\frac{p}{r}$ observed variables in this way.

In terms of the matrix $\mathcal{A}^*$, the overall (observed + latent) system has the form given by the matrix below

\begin{figure}[h]
\centering
\includegraphics[width=0.28\linewidth]{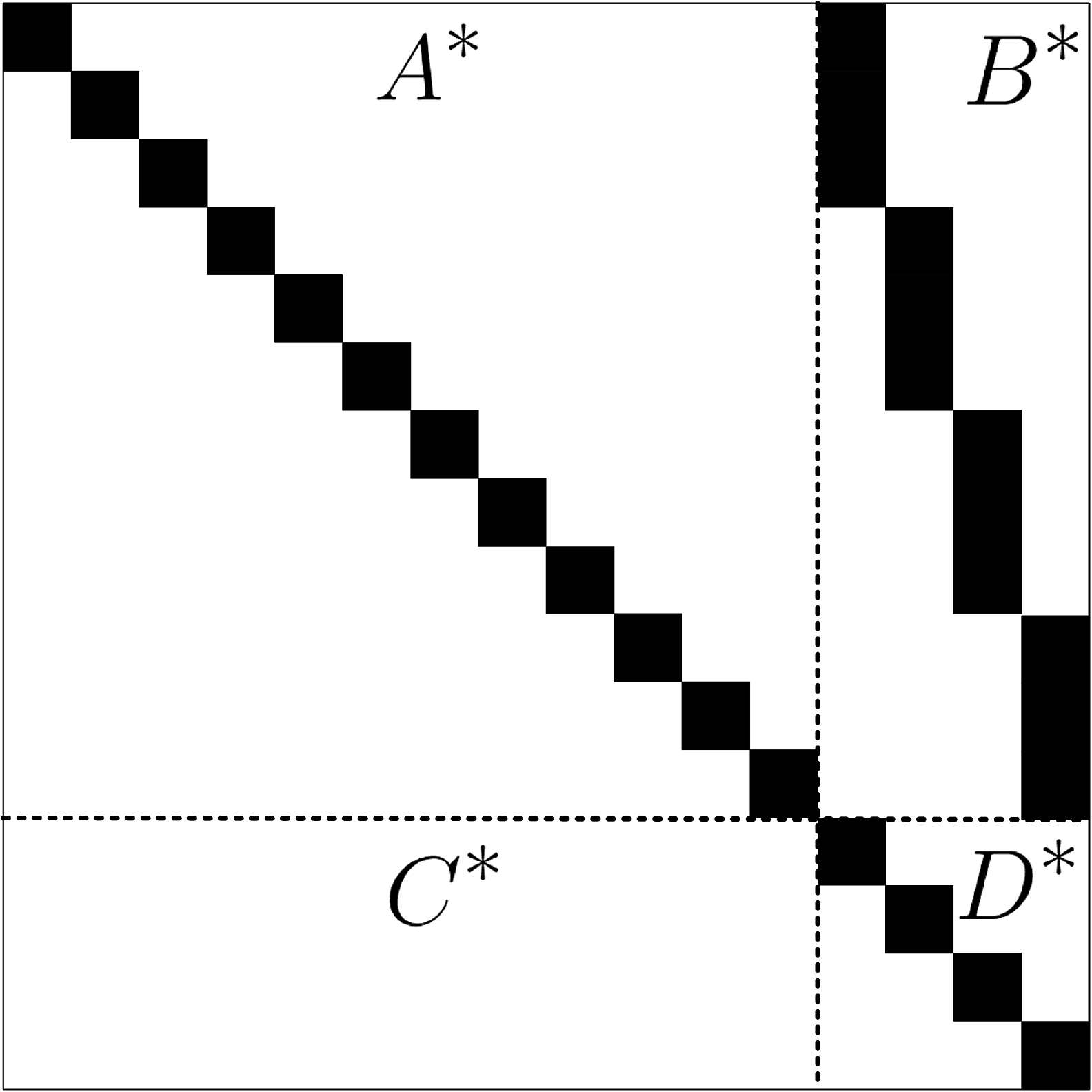}
\end{figure}

Here $A^* = - I_{p\times p}$, $C^* = 0$, and $D^* = - I_{r\times r}$. This matrix satisfies stability assumption {\bf (A1)}. In the matrix $B^*$ , each column has exactly $\frac{p}{r}$ entries that are $1$, and the remaining are $0$. Each row of $B^*$ has exactly one entry that is $1$, and the remaining are $0$; note that the columns of $B$ are orthogonal. We start from zero initial condition with $\eta=0$ (continuous time system). With this, $D=2$ and $\|B^*\|_{\infty,1}=1$.  

For this idealized setting, we can exactly evaluate all the quantities we need. In particular, it is not hard to show (done in Appendix) that the steady-state covariance matrices are $Q^* = \frac{1}{2}(I + BB^T)$ and $R^* = B^{*T}$. The resulting low-rank matrix is $L^* = \frac{r}{p+r} BB^T$, which gives $U = V = \sqrt{\frac{r}{p}}B$; the incoherence parameter $\mu = r$, and hence we need $r < \sqrt{p} / 3$ by assumption {\bf (A2)}. Moreover, we can show that $\theta =\frac{1}{2}$ for this example and hence the assumption {\bf (A3)} is also satisfied.  

Similarly. evaluating the other parameters in Theorem \ref{thr:main}, we get that the observation time should be $T \geq K s^3 \log \frac{4(1+2r)p + 4r^2}{\delta}$ for structure recovery with probability greater than $\delta$. In this case, we also have $\nu =\frac{3r}{4\sqrt{p}}+\frac{25\sqrt{s}}{7}$ and $\rho_0 = \frac{1}{5+\frac{32\sqrt{p}}{3r\lambda_A}}$ providing the error bounds $\|A^*-\hat{A}\|_\infty\leq \left(\frac{3r}{4\sqrt{p}}+\frac{25\sqrt{s}}{7}\right)\lambda_A$ and $\|L^*-\hat{L}\|_2\leq \frac{3r}{32\sqrt{p}}\lambda_A$.

\section{Proof of the Theorem}
\label{sec:proof}
\noindent In this section, we first introduce some notations and definitions and then, provide a three step proof technique to prove the main theorem for the discrete time system. The proof of the continuous time system is done via a coupling argument in the appendix.

Before we proceed to the details, we would like to make a high level technical remark on the novelties of our proof. There are two key novel ingredients in the proof enabling us to get the low sample complexity result in our theorem. The first ingredient comes from our new set of optimality conditions inspired by \cite{CANLIMAWRI}. This optimality conditions enable us to certify an approximation of $L^*$ while certifying the exact sign support of $A^*$. The second ingredient comes from the bounds on the Schur complement of the perturbation of positive semi-definite matrices \cite{STE}. This result enables us to get a bound on the Schur complement of a perturbation of a positive semi-definite matrix of size $p$ with only $\log(p)$ samples. 

Given a matrix $A^*$, let $\O$ be the subspace of matrices whose their support is a subset of the matrix $A^*$. The orthogonal projection of a matrix $M$ to $\O$ is denoted by $\P_\O(M)$. Denote the orthogonal complement space with $\O^c$ with orthogonal projection $\P_{\O^c}(M)$.

For any matrix $L\in\real^{p\times p}$, if the SVD is $L = U\Sigma V^T$, then let $\T(L) := \{M | M = UX^T + YV^T \text{for some $X,Y$}\}$ denote the subspace spanned by all matrices that have the same column space or row space as $L$. The orthogonal projection of a matrix $N$ to $\T$ is denoted by $\P_\T(N)$. Denote the orthogonal complement space with $\T^c$ with orthogonal projection $\P_{\T^c}$. We define a metric to measure the \emph{closeness} of two subspaces $\T_1$ and $\T_2$ as follows
$$\rho\left(\T_1,\T_2\right) = \max_{N\in\real^{p\times p}}\, \frac{\|\P_{\T_1}(N) - \P_{\T_2}(N)\|_2}{\|N\|_2}.$$ Finally, let $\T=\T(L^*)$ to shorten the notation and $L^*=U^*\Sigma^*V^*$ be a singular value decomposition.

The proof steps are as follows:
\begin{itemize}
\item {\bf STEP 1:} We construct a candidate primal optimal solution $(\tA,\tL)$ with the desired sparsity pattern using the restricted support optimization problem. We refer to this as \emph{oracle problem}:
\small\begin{equation}
\begin{aligned}
(\tA,\tL)\! &=\arg \min_{L:\rho(\T(L),\T)\leq\rho_0\atop{A:\P_{\O^c}(A)=0}}\frac{1}{2\eta^2n}\sum_{i=0}^{n-1} \left\|x(i+1)\!-\!x(i)\!-\!\eta(A+L)x(i)\right\|_2^2\\ &\qquad\qquad\qquad\qquad\qquad\qquad\qquad\qquad\qquad\qquad+\lambda_A\|A\|_1\!+\!\lambda_L\|L\|_*.
\end{aligned}
\label{eq:opt-oracle}
\end{equation}\normalsize
This oracle is similar to the one used in \cite{CHAPARWIL}. Note that this is a proof technique, not a method to construct the solution.\\

\item {\bf STEP 2:} We Write down a novel set of sufficient (stationary) optimality conditions for $(\tA,\tL)$ to be the unique solution of the (unrestricted) optimization problem~\eqref{eq:opt-orig}:

\begin{lemma}\label{LemWitnessOptCond}
If $\O\cap\T=\{0\}$, then $(\tA,\tL)$, the solution to the oracle problem~\eqref{eq:opt-oracle}, is the unique solution of the problem \eqref{eq:opt-orig} if there exists a matrix $\EstDual\in\real^{p\times p}$ such that\\

\small
\noindent {\bf (C1)} $\P_\O(\EstDual)=\lambda_A\sgn\left(\tA\right)$. $\qquad\qquad\qquad${\bf (C2)} $\left\|\P_{\O^c}(\EstDual)\right\|_{\infty} < \lambda_A$.\\

\noindent {\bf (C3)} $\left\|\P_\T(\EstDual)-\lambda_LU^*V^{*T}\right\|_2 \leq 4\rho\lambda_L$. $\qquad${\bf (C4)} $\left\|\P_{\T^c}(\EstDual)\right\|_2< (1-\alpha)\lambda_L$.\\

\noindent {\bf (C5)} $-\underbrace{\frac{1}{\eta n} \sum_{i=1}^n \left(x(i+1)-x(i)-\eta(\tA+\tL)x(i)\right)x(i)^T}_{J_n} + \EstDual = 0$.\normalsize\\
\end{lemma}

Upon existence of $\EstDual$, the solution of the oracle problem not only is the solution to the original problem~\eqref{eq:opt-orig}, but also satisfies the claim of the theorem. 
\begin{lemma}
Provided $\EstDual$ in Lemma~\ref{LemWitnessOptCond} exists, we have
\begin{itemize}
\item [(a)] $\supp(\tA) \subset \supp(A^*)$.
\item [(b)] $\|\tA - A^*\|_\infty\leq \nu \lambda_A$ and $\|\tL - L^*\|_2\leq \frac{\rho_0}{1-5\rho_0}\|L^*\|_2$
\item [(c)] If $A_{min} > \nu \lambda_A$ then $\sgn(\tA) = \sgn(A^*)$.
\end{itemize}
\label{lem:tildecond}
\end{lemma}
Part (a) is immediate by the constraints of the oracle problem and provided the $\ell_\infty$ bound in (b) and part (a), the result of part (c) naturally follows. We prove part (b) in the Appendix~\ref{sec:lemtilde}. Now, it suffices to construct a dual variable $\EstDual$.\\

\item{\bf STEP 3:} Constructing a dual variable $\EstDual$ that satisfies the sufficient optimality conditions stated in Lemma~\ref{LemWitnessOptCond}. First notice that under assumption {\bf (A2)}, we have $\O\cap\T=\{0\}$ \cite{CHEJALSANCAR}. For matrices $M\in\O$ and $N\in\T$, let
\small\begin{equation}
\begin{aligned}
\H_M &= M - \P_\T(M) + \P_\O\P_\T(M) - \P_\T\P_\O\P_\T(M) + \ldots\\
\G_N &= \,N - \,\P_\O(N) + \,\P_\T\P_\O(N) - \,\P_\O\P_\T\P_\O(N) + \ldots.
\end{aligned}
\nonumber
\end{equation}\normalsize
It has been shown in \cite{CHEJALSANCAR} that if $\alpha<1$ then both infinite sums converge. Suppose we have the SVD decomposition $\tL=\tU\tSigma\tV^T$. Let
\begin{equation}
\EstDual = \H_{\lambda_A\sgn(\tA)} + \G_{\P_\T(\lambda_L\tU\tV^T)} + \Delta,
\nonumber
\end{equation}
where, $\Delta$ is a matrix such that {\bf (C5)} is satisfied. As a result of our construction, we have $\P_\O(\EstDual-\Delta)=\lambda_A\sgn(\tA)$ and by optimality of $(\tA,\tL)$, we have $\P_\O(J_n)=\lambda_A\sgn(\tA)$. This entails that $\P_\O(\Delta)=0$ and hence {\bf (C1)} is satisfied.\\

To show {\bf (C3)} holds, we need the next lemma.
\begin{lemma}
$\P_\T(J_n)=\P_\T(\lambda_L\tU\tV^T)$.
\label{lem:TtTConnection}
\end{lemma}
By our construction, we have $P_{\T}(\EstDual-\Delta)=\P_\T(\lambda_L\tU\tV^T)=\P_\T(J_n)$ by Lemma~\ref{lem:TtTConnection}. Consequently, $\P_\T(\Delta)=0$ and hence {\bf (C3)} is also satisfied, considering the oracle constraint bound $\rho_0$.\\

It suffices to show that {\bf (C2)} and {\bf (C4)} are satisfied with high probability. This has been shown in the next Lemma.
\begin{lemma}
Under assumptions {\bf (A1)-(A5)},  $\EstDual$ satisfies conditions {\bf (C2)} and {\bf (C4)} with probability $1-c_1\exp(-c_2n)$ for some positive constants $c_1$ and $c_2$.
\label{lem:dual-feasibility}
\end{lemma}

This concludes the proof of the theorem for the discrete time system.\\

\item{\bf STEP 4:} Denote $X(t)=[x(t)\,u(t)]^T$ and let 
\begin{equation}
\widehat{\mathcal{Q}}=\frac{1}{T}\int_{t=0}^TX(t)X(t)^Tdt \qquad \widehat{\mathcal{W}}=\frac{1}{T}\int_{t=0}^Tdw(t)X(t)^T.
\nonumber
\end{equation}
Having the result for the discrete time system, it suffices (see proof of Theorem 1.1 in \cite{BENIBRMON} for more details) to show that for a given continuous time system, there exists a discrete time system with $\mathcal{Q}^{(n)}$ and $\mathcal{W}^{(n)}$ such that almost surely,
\begin{equation}
\mathcal{Q}^{(n)}\longrightarrow\widehat{\mathcal{Q}} \qquad\qquad\qquad \mathcal{W}^{(n)}\longrightarrow\widehat{\mathcal{W}},
\nonumber
\end{equation}
as $n\rightarrow\infty$ for a fixed $T=n\eta$ (and hence, $\eta\rightarrow 0$).

Let $\mathcal{Q}^*$ be the matrix satisfying the continuous time Lyapunov stability equation $\mathcal{A}^*\mathcal{Q}^*+\mathcal{Q}^*\mathcal{A}^{*T}+I=0$ and $\mathcal{Q}^*(\eta)$ be the matrix satisfying the discrete time Lyapunov stability equation $\mathcal{A}^*\mathcal{Q}^*(\eta)+\mathcal{Q}^*(\eta)\mathcal{A}^{*T} +\eta\mathcal{A}^*\mathcal{Q}^*(\eta)\mathcal{A}^{*T}+I=0$. It is easy to see that $\mathcal{Q}^*(\eta)\rightarrow\mathcal{Q}^*$ as $\eta\rightarrow 0$ by the uniqueness of the stationary distribution. Moreover, by Lemma~\ref{lem:Q-bound}, we know that $\mathcal{Q}^{(n)}\rightarrow\mathcal{Q}^*(\eta)$ as $n\rightarrow\infty$.

Now, let the initial state of the discrete time system be $$X(i=0)=\left(\mathcal{Q}^*(\eta)\right)^{1/2}\left(\mathcal{Q}^*\right)^{-1/2}X(t=0),$$ and the noise $w(i)=\,w(t=i\eta)-w(t=(i-1)\eta)$. It can be easily checked that $w(i)\sim\mathcal{N}(0,\eta I)$ if the continuous time $w(t)$ is a Brownian motion. Thus, $x(i)$ and $x(t)$ are coupled and the almost sure convergence, follows from the convergence of random walks to Brownian motions \cite{MAR}. This concludes the proof of the theorem for continuous time systems.\\

\end{itemize}

\begin{figure}
\centering
\subfigure[{\small Effect of $\eta$}]{
\includegraphics[width=3.5in]{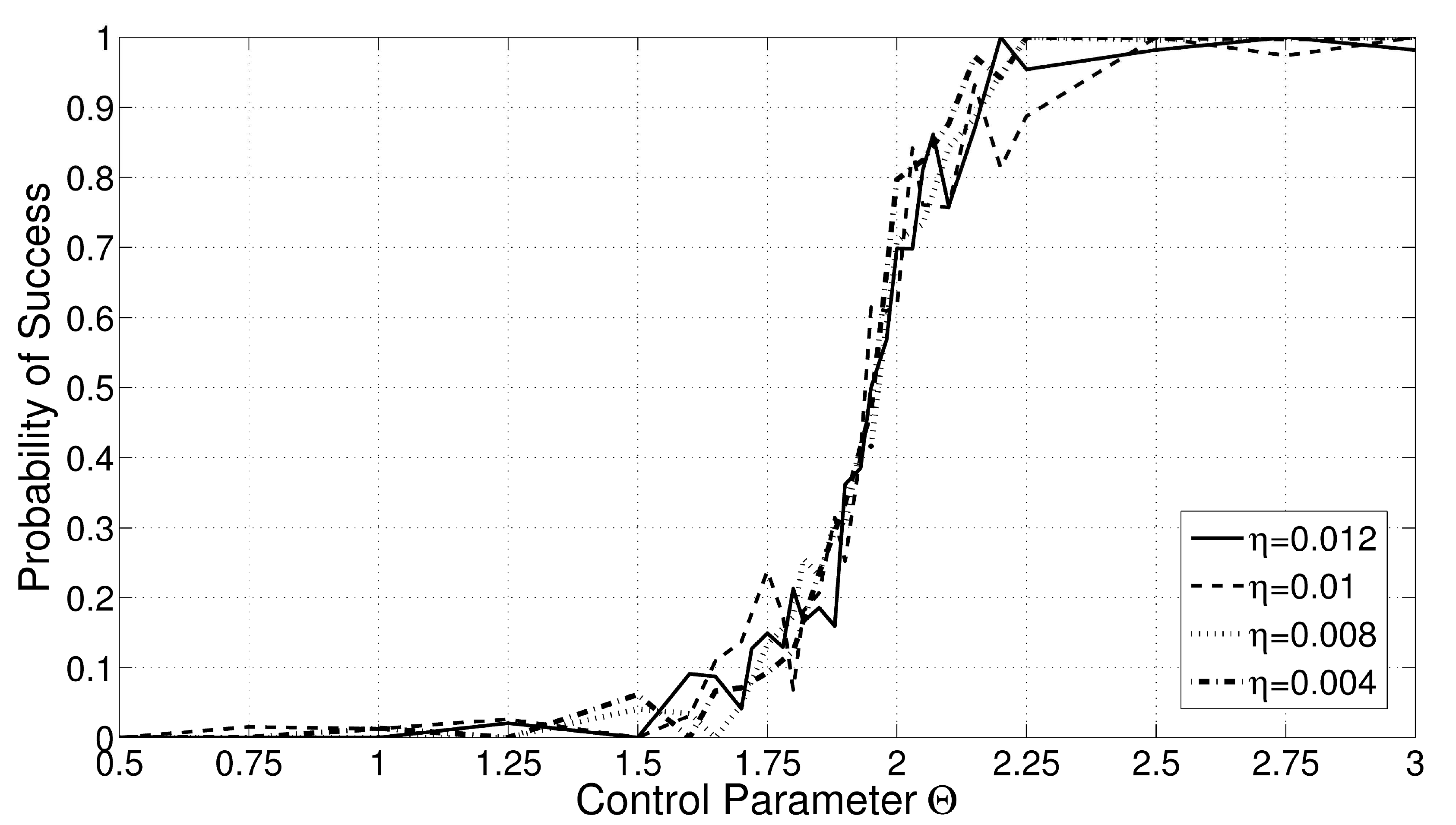}
\label{fig1:eta}
}
\subfigure[{\small Effect of $r$}]{
\includegraphics[width=3.5in]{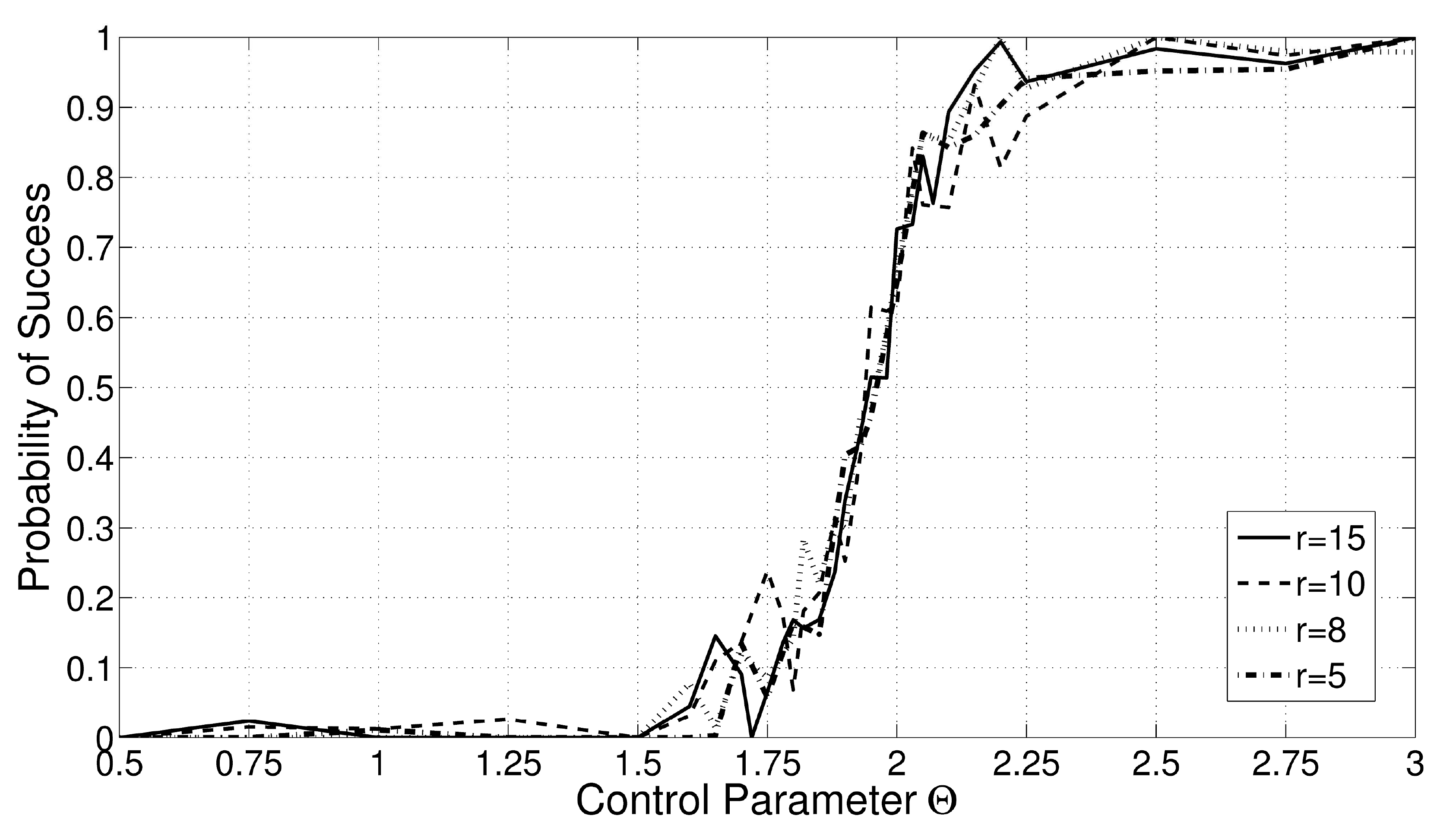}
\label{fig1:r}
}
\subfigure[{\small Effect of $s$}]{
\includegraphics[width=3.5in]{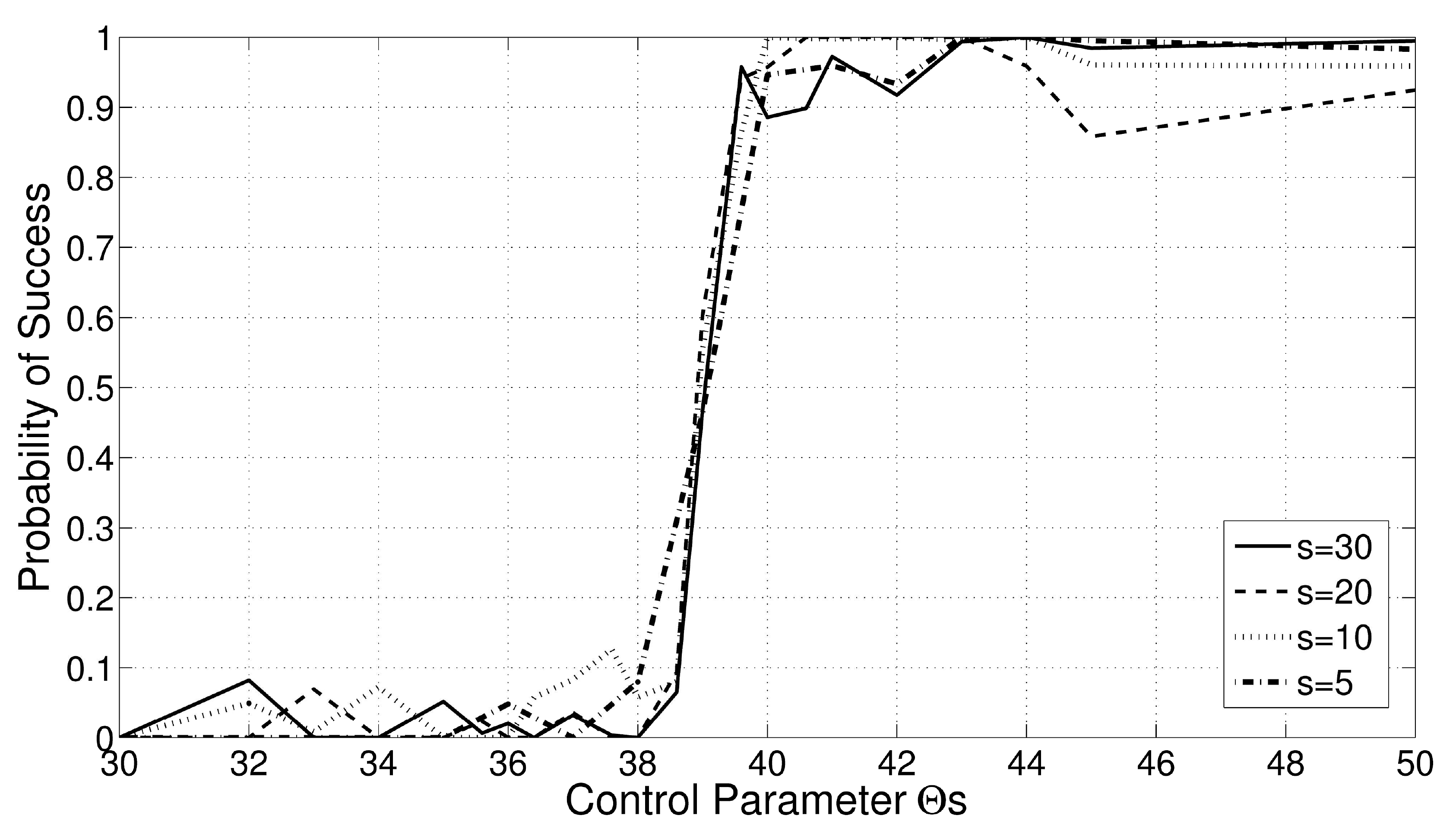}
\label{fig1:s}
}
\label{fig1}
\caption{\small Probability of success in recovering the true signed support of $A^*$ versus the control parameter $\Theta$ with $p=200$, $r=10$ and $s=20$ for different values of $\eta$ in \ref{fig1:eta}, and, with $p=200$, $s=20$ and $\eta=0.01$ for different number of latent time series $r$ in \ref{fig1:r}, and, with $p=200$, $r=10$ and fixed $\eta=0.01$ for different sparsity sizes $s$ in \ref{fig1:s}. Notice that Fig.~\ref{fig1:s} is plotted versus $\Theta\times s$ which means $n\eta$ scales with $s^2$ not $s^3$. This means our theoretical result can be tightened.}
\end{figure}

\section{Experimental Results}
\label{sec:exper}

\subsection{Synthetic Data}

Motivated by the example discussed in the paper, we simulate a similar (but different) dynamic system for the purpose of our experiments. Consider the system where each latent variable is only evolves by itself, i.e., $C^*=0$ and $D^*$ is a diagonal matrix. Moreover, assume that each observed variable is affected by exactly two latent variable, i.e., each column of $B^*$ has $2p/r$ non-zeros and each row of $B^*$ has two non-zeros. We randomly select a support of size $s$ per row for $A^*$ and draw all the values of $A^*$ and $B^*$ i.i.d. standard Gaussian. To make the matrix $\mathcal{A}^*$ negative definite (hence, stable), using Ger{\v s}gorin disk theorem \cite{GER}, we put a large-enough negative value on the diagonals of $A^*$ and $D^*$. 

We generate the data according to the continuous time model. The solution to the first order system can be written as
\begin{equation}
\left[\begin{array}{c} x(t)\\ u(t)\end{array}\right] = e^{\mathcal{A}^*(t-t_0)}\left[\begin{array}{c} x(t_0)\\ u(t_0)\end{array}\right] + \int_{t_0}^t e^{\mathcal{A}^*(t-\tau)}dw(\tau),
\nonumber
\end{equation}
where, $e^{\mathcal{A}^*} = I + \mathcal{A}^* + \frac{1}{2}\mathcal{A}^{*2} + \ldots$ is a generalization of the exponential function to matrices. We sub-sample this system at points $t_i = \eta i$ for $i=1,2,\ldots,n$, that is
\begin{equation}
\left[\begin{array}{c} x(i)\\ u(i)\end{array}\right] = e^{\eta \mathcal{A}}\left[\begin{array}{c} x(i-1)\\ u(i-1)\end{array}\right] + \int_{\eta (i-1)}^{\eta i} e^{\mathcal{A}(\eta i-\tau)}dw(\tau)
\nonumber
\end{equation}
The stochastic integral can be estimated by binning the interval and assuming the Brownian motion is constant over the bin and hence, can be estimated by a standard Gaussian. For more information on this integration method, we refer to Chapter 4 of \citet{SHR04}.

Using this data, we solve \eqref{eq:opt-orig} using accelerated proximal gradient method \cite{LINGANWRIWUCHEMA}. Motivated by our Theorem, we plot our result with respect to the control parameter $\Theta=\frac{\eta n}{s^3\log\left((s+2r)p+r^2\right)}$. We pick the values of $\lambda_A$ and $\lambda_L$ by dividing the training data into chunks each having consecutive samples and do the cross validation over those chunks. Note that this is different from the standard cross validation technique due to the dependency of samples. 

Figure~\ref{fig1} shows the phase transition of the probability of success in recovering the exact sign support of the matrix $A^*$. We ran three different experiments, each investigating the effect of one of the three key parameters of the system $\eta$ (sampling frequency), $r$ (number of latent variables) and $s$ (sparsity of the model). These three figures show that the probability of success curves line up if they are plotted versus the correct control parameter. The first two curves for $\eta$ and $r$ line up versus $\Theta$, indicating that our theorem suggests the correct scaling law for the sample complexity. However, from this experiment, it seems that the phase transition probability scales with $s^2$ not $s^3$. Perhaps the result of our theorem and also \citet{BENIBRMON} (for $r=0$) can be tightened.  

\subsection{Stock Market Data}
We take the end-of-the-day closing stock prices for 50 different companies in the period of May 17, 2010 - May 13, 2011 ($255$ business days). These companies (among them, Amazon, eBay, Pepsi, etc) are consumer goods companies traded either at NASDAQ or NYSE in USD. The data is collected from Google Finance website. Our goal is to observe the stock prices for a period of time and predict it for the entire days of the next month with small error.

Applying our method and pure LASSO \cite{BENIBRMON} to the data, we recover the structure of the dependencies among stocks. We represent the result as a graph in Fig~\ref{fig3}; where each company is a node in this graph and there is an edge between company $i$ and $j$ if $\hat{A}_{ij}\neq 0$. This result shows that the recovered dependency structure by our algorithm is order of magnitude sparser than the one recovered by pure LASSO. 

\begin{figure}[t]
\centering
\subfigure[{\small Pure LASSO}]{
\includegraphics[width=2.3in]{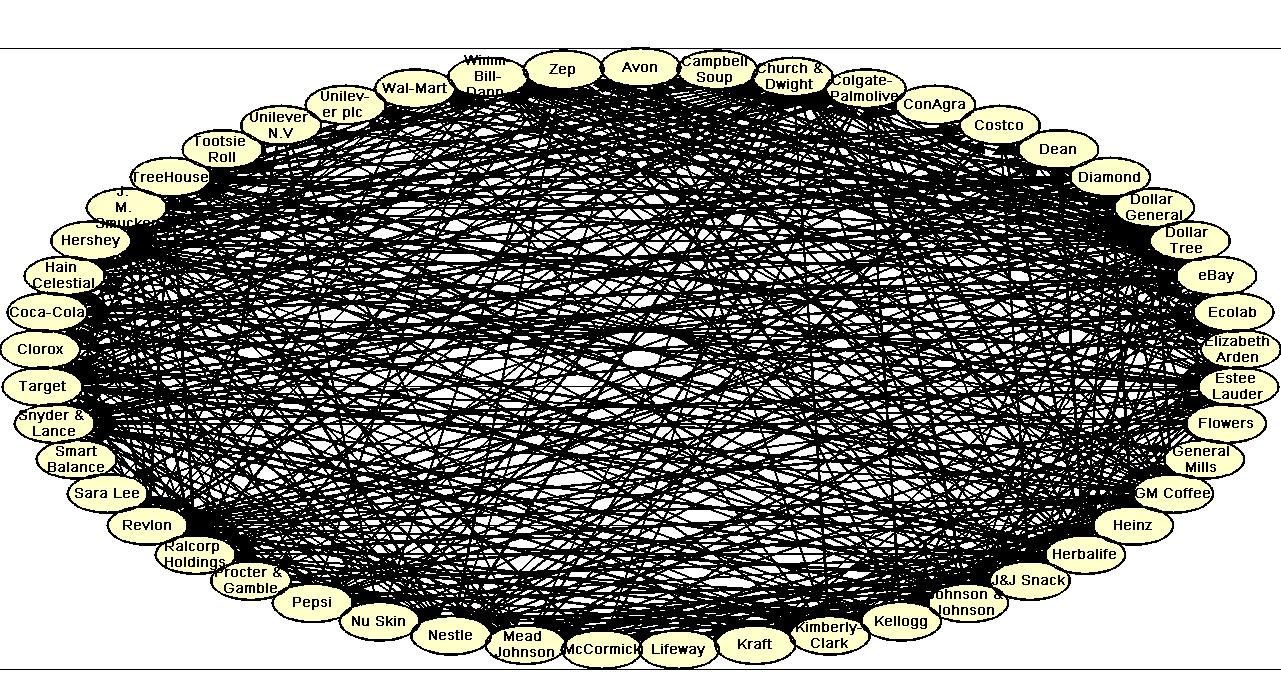}
\label{fig3:LASSO}
}
\subfigure[{\small Our Algorithm}]{
\includegraphics[width=2.3in]{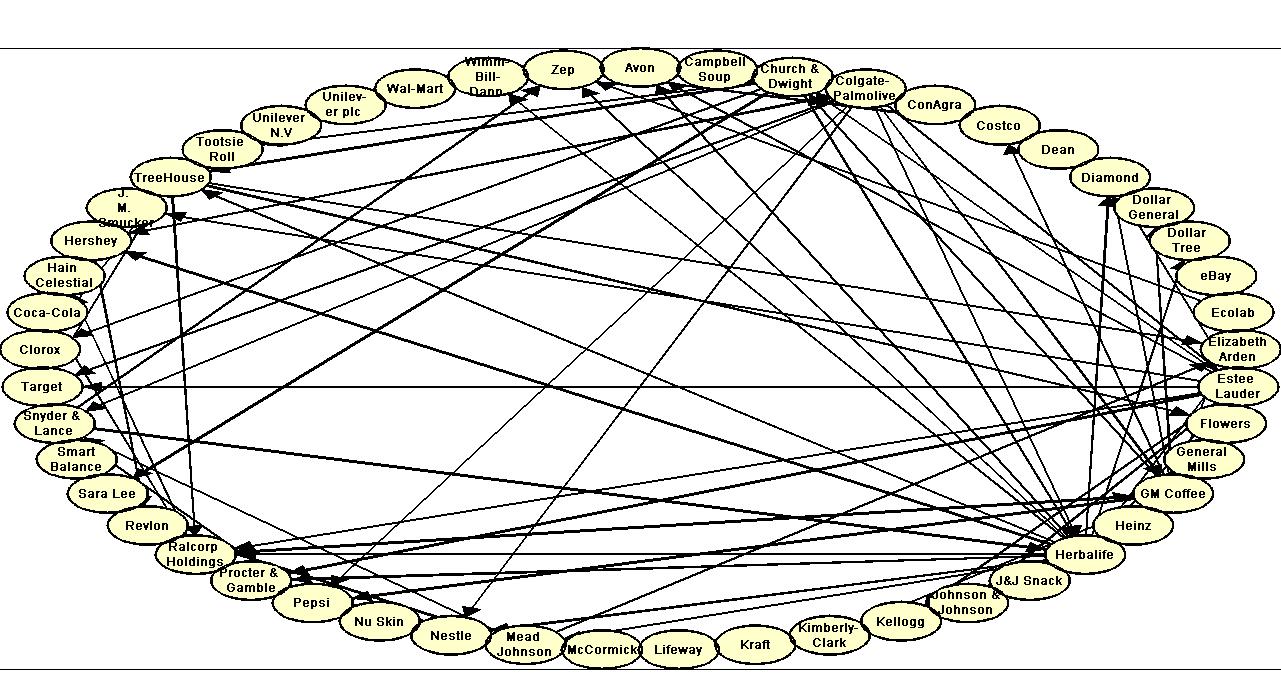}
\label{fig3:Ours}
}
\label{fig3}
\caption{\small Comparison of the stock dependencies recovered by Pure LASSO \cite{BENIBRMON} and our algorithm.}
\end{figure}

\begin{figure}[t]
\centering
\subfigure[{\small Model Sparsity}]{
\includegraphics[width=2.3in]{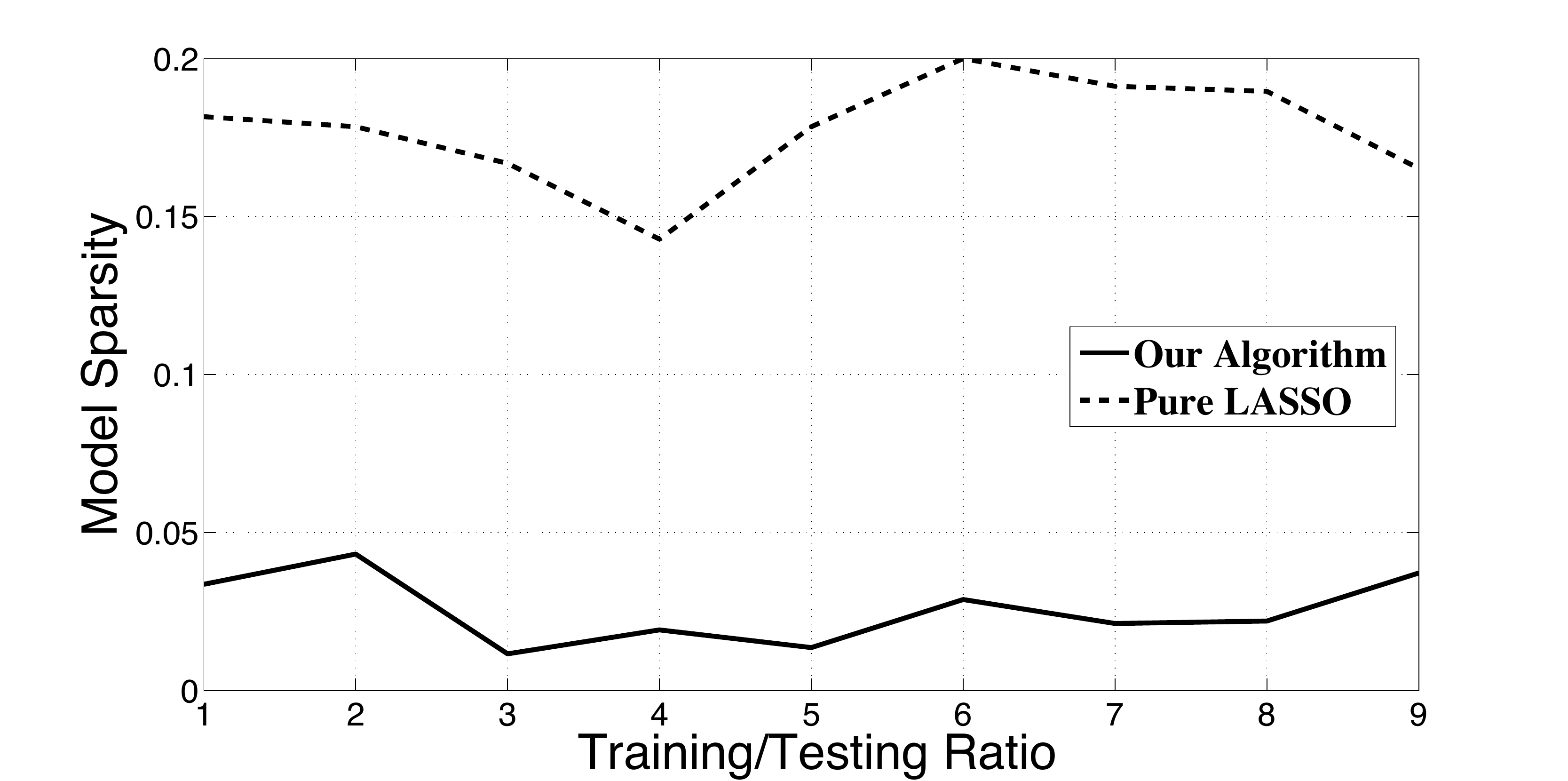}
\label{fig2:eta}
}
\subfigure[{\small Prediction Error}]{
\includegraphics[width=2.3in]{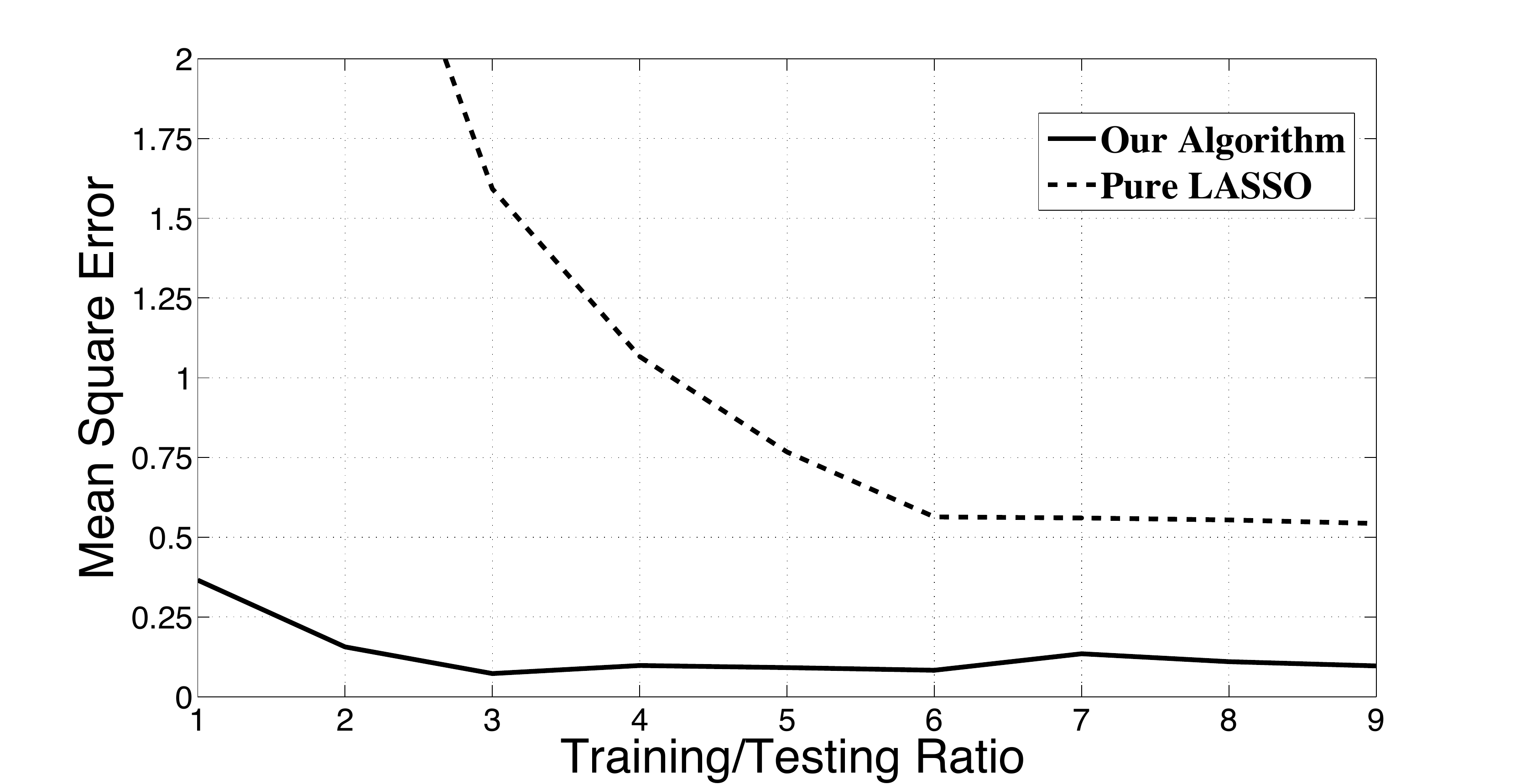}
\label{fig2:r}
}
\label{fig2}
\caption{\small Prediction error and model sparsity versus the ratio of the training/testing sample sizes for prediction of the stock price. Prediction error is measured using mean squared error and the model sparsity is the number of non-zero entries divided by the size of $\hA$.}
\end{figure}

To show the usefulness of our algorithm for prediction purposes, we apply our algorithm to this data and try to learn the model using the data for $n$ (consecutive) days and then compute the mean squared error in the prediction of the following month ($25$ business days). We randomly pick an starting day $n_0$ between day $1$ and day $255-25-n$. Then we learn the model using the data from the day $n_0$ to the day $n_0+n$ (total of $n$ days). Then, we test our data on the consecutive $25$ days. Finally, we average the error over $10$ different starting points $n_0$ for each value of $n$. We pick the regularizers by the semi-cross validation process explained in the previous section. The ratio $\frac{n}{25}$ shows the ratio of training sample size to the testing sample size. 

Figure~\ref{fig2:r} shows the prediction error for both our method and pure LASSO \cite{BENIBRMON} method as the train/test ratio increases. It can be seen that our method not only have better prediction, but also is more robust. Our algorithm requires only 3 months of the past data to give a robust estimation of the next month; in contrast with almost 6 months requirement of LASSO. However, the error of our algorithm is much smaller (by a factor of 6) than LASSO even in the steady state. Figure~\ref{fig2:eta} shows the sparsity level for our model and the LASSO model. The number of latent variables our model finds varies from $8-12$ for different train/test ratios. As Figure~\ref{fig2:eta} illustrates, our estimated $\hA$ is order of magnitude sparser than the one estimated by LASSO.

\appendix
\section{Proof of Lemma~\ref{lem:idea}}
Ignoring the term $\|x(n+1)-x(n)\|_2^2$ which is independent of $A$, minimization of $\L(A)$ with this infinite sample size is equivalent to
\begin{equation}
\begin{aligned}
&\!\!\!\!\min_{A}\,\mathbb{E}\left[x(n)^TA^TAx(n) - \frac{2}{\eta}(x(n+1)-x(n))^TAx(n)\right]\\ &\!\!\!\!= \min_{A}\,\mathbb{E}\left[\trace{Ax(n)x(n)^T\!\!A^T}\!\! -\! 2\left(A^*x(n)\!\!+\!B^*u(t)\right)^T\!\!\!Ax(n)\right]\\
&\!\!\!\!= \min_{A}\,\trace{AQ^*A^T}\!\! -\! 2\trace{A^*Q^*A^T}\!\! -\!2\trace{B^*R^*A^T}\\
&\!\!\!\!= \min_{A}\,\trace{\Big(A -2\left(A^* + B^*R^*(Q^*)^{-1}\right)\Big)Q^*A^T}.
\end{aligned}
\nonumber
\end{equation}
Here we ignored the term $w(n)$ due to the fact that it is zero mean and independent of $x(n)$ and $u(n)$. This implies that the asympotatic optimizer of $\L(\cdot)$ satisfies $\hat{A} = A^* + B^*R^*(Q^*)^{-1}$.  This concludes the proof of the lemma.

\section{Illustrative Example}
\label{sec:example}
In this section, we analyze the illustrative example discussed in Sec~\ref{sec:main}. For that example, Lyapunov stability equation requires
\begin{equation}
\left[\begin{array}{ccc} -2Q^*+B^*R^*+R^{*T}B^{*T} &\qquad& -2R^{*T}+B^*\\ -2R^*+B^{*T}& & -2P^*\end{array}\right] = \left[\begin{array}{ccc} -I && 0\\ 0& & -I\end{array}\right] .
\nonumber
\end{equation}
This entails that $R^* = \frac{1}{2} B^{*T}$ and $Q^* = \frac{1}{2}(I+B^*B^{*T})$ with $C_{\min}=\frac{1}{2}$. It can be easily checked that $Q^{*-1} = 2(I - \frac{r}{p+r}B^*B^{*T})$. Thus, the low-rank matrix of interest is
\begin{equation}
\begin{aligned}
L^* &= B^*R^*Q^{*-1}\\
&= B^*B^{*T}(I - \frac{r}{p+r}B^*B^{*T})\\
&= (1-\frac{p}{p+r})B^*B^{*T}.
\end{aligned}
\nonumber
\end{equation}
Taking singular value decomposition $U^*\Sigma^* V^*$ of this matrix, we get $U^*=V^*=\sqrt{\frac{r}{p}}B^*$ and hence $\mu=r$. Considering $s=1$, the identifiability assumption {\bf (A2)} becomes $\alpha =\frac{3r}{\sqrt{p}}\leq 1$ or equivalently, $r\leq\frac{\sqrt{p}}{3}$.

Considering assumption {\bf (A3)}, note that $Q^*_{\S_\svert\S_\svert}=1$ is just an scalar since $s=1$. Moreover, $Q^*_{\S_{\svert^c}\S_\svert}$ is a vector with all entries equal to $\frac{1}{2}$ and hence $\theta = \frac{1}{2}$.

\section{Proof of Lemma~\ref{LemWitnessOptCond}}
\label{sec:opt-cond}

Suppose $D_A = \tA - \hA$ and $D_L = \tL - \hL$. From condition {\bf (C5)}, $-\EstDual$ is the subgradient of the loss function at $(\tA,\tL)$ and hence,
\begin{equation}
\L(\hA+\hL) \geq \L(\tA+\tL) - \tr{-\EstDual}{D_A + D_L}. 
\label{eq:n1}
\end{equation}

Let $Z_A = \lambda_A\sgn(\tA) - F$ with $F=\lambda_A\sgn(\P_{\O^c}(D_A))$. Notice that $\P_\O(F)=0$ and $\tr{F}{D_A} =\lambda_A\|\P_{\O^c}(D_A)\|_1$. For $Z_A$ is in the subgradient of $\lambda_A\|\tA\|_1$, we have
\begin{equation}
\begin{aligned}
\lambda_A\|\hA\|_1 \geq \lambda_A\|\tA\|_1 - \tr{Z_A}{D_A}. 
\end{aligned}
\label{eq:n2}
\end{equation}

Suppose $\tL=\tU\tSigma\tV^T$ and $\P_{\T^c}(D_L)=U_D\Sigma_DV_D$ are SVD decompositions. Now, let $Z_L = \lambda_LU^*V^{*T} + W_1 - W_0$ with $W_0=(1-\alpha)\lambda_LU_DV_D$ and $W_1=\P_\T(\lambda_L\tU\tV^{T})-\lambda_LU^*V^{*T}$. In this construction, we have\\

\begin{itemize}
\item [(a)] $\P_{\T}(W_0)=0$ and $\|W_0\|_2\leq(1-\alpha)\lambda_L$ and $\tr{W_0}{D_L} =(1-\alpha)\lambda_L\|\P_{\T^c}(D_L)\|_*$.\\

\item [(b)] $\P_{\T^c}(W_1)=0$ and $\|W_1\|_2\leq 4\rho_0\lambda_L$ by Lemma~\ref{lem:W-existence}.\\
\end{itemize}

Let $W_2=-\P_{\T^c}(\lambda_L\tU\tV^T) - W_0$ and notice that $Z_L = \lambda_L\tU\tV^T + W_2$. Here, we have $\P_{\tT^c}(W_2)=0$ and $\|W_2\|_2<\lambda_L$ by Lemma~\ref{lem:W-existence}. Hence, our constructed $Z_L$ is in the subgradient of $\lambda_L\|\tL\|_*$, i.e.,
\begin{equation}
\begin{aligned}
\lambda_L\|\hL\|_* \geq \lambda_L\|\tL\|_* - \tr{Z_L}{D_L}. 
\end{aligned}
\label{eq:n3}
\end{equation}

Combining \eqref{eq:n1}-\eqref{eq:n3}, we get
\begin{equation}
\begin{aligned}
\L(\hA+\hL) + \lambda_A\|\hA\|_1 + \lambda_L\|\hL\|_*&\geq \L(\tA+\tL) + \lambda_A\|\tA\|_1 + \lambda_L\|\tL\|_*\\
&+ \tr{\EstDual}{D_A + D_L} - \tr{Z_A}{D_A} - \tr{Z_L}{D_L}. 
\end{aligned}
\nonumber
\end{equation}
Provided $\tr{\EstDual}{D_A + D_L} - \tr{Z_A}{D_A} - \tr{Z_L}{D_L}\geq 0$, we arrive to a contradiction with the optimality of $(\hA,\hL)$ and the result follows.

Notice that by first order optimality condition, we have $\P_\O(\EstDual)=\lambda_A\sgn(\tA)=\P_\O(Z_A)$. Hence, for some $\gamma<1$, we have
\begin{equation}
\begin{aligned}
\tr{\EstDual}{D_A} &- \tr{Z_A}{D_A}\\ &= \tr{\P_{\O^c}(\EstDual)}{\P_{\O^c}(D_A)} + \tr{F}{\P_{\O^c}(D_A)} + \underbrace{\tr{\P_{\O}(\EstDual)-\P_\O(Z_A)}{\P_\O(D_A)}}_{=0\quad\text{by {\bf (C1)}}}\\
&= \tr{\P_{\O^c}(\EstDual)}{\P_{\O^c}(D_A)} + \lambda_A\|\P_{\O^c}(D_A)\|_1\qquad \text{(Construction of $F$)}\\
&\geq -\gamma\lambda_A\|\P_{\O^c}(D_A)\|_1 + \lambda_A\|\P_{\O^c}(D_A)\|_1\qquad \text{(by {\bf (C2)})}\\
&=(1-\gamma)\lambda_A\|\P_{\O^c}(D_A)\|_1.
\end{aligned}
\label{eq:n4}
\end{equation}

Similarly, by first order optimality condition, $\P_{\tT}(\EstDual)=\lambda_L\tU\tV^T$ 
and by our construction, $\P_{\T}(\EstDual)=\lambda_LU^*V^{*T}+W_1=\P_\T(Z_L)$. Hence, we get
\begin{equation}
\begin{aligned}
\tr{\EstDual}{D_L} &- \tr{Z_L}{D_L}\\ &= \tr{\P_{\T^c}(\EstDual)}{\P_{\T^c}(D_L)} + \tr{W_0}{\P_{\T^c}(D_L)} + \underbrace{\tr{\P_{\T}(\EstDual)-\P_\T(Z_L)}{\P_\T(D_L)}}_{=0\quad\text{by {\bf (C3)}}}\\
&= \tr{\P_{\T^c}(\EstDual)}{\P_{\T^c}(D_L)} + (1-\alpha)\lambda_L\|\P_{\T^c}(D_L)\|_*\qquad \text{(Construction of $W_0$)}\\
&\geq -\gamma(1-\alpha)\lambda_L\|\P_{\T^c}(D_L)\|_* + (1-\alpha)\lambda_L\|\P_{\T^c}(D_L)\|_*\qquad \text{(by {\bf (C4)})}\\
&=(1-\gamma)(1-\alpha)\lambda_L\|\P_{\T^c}(D_L)\|_*.
\end{aligned}
\label{eq:n5}
\end{equation}

Combining \eqref{eq:n4} and \eqref{eq:n5}, we get
\begin{equation}
\begin{aligned}
\tr{\EstDual}{D_A + D_L} - \tr{Z_A}{D_A} - \tr{Z_L}{D_L}\geq 0.\\
\end{aligned}
\nonumber
\end{equation}

This concludes the proof of the lemma.\\

\begin{lemma}
For $W_1$ and $W_2$ constructed above, we have $\|W_2\|_2<\lambda_L$ and $\|W_1\|_2\leq 4\rho\lambda_L$.
\label{lem:W-existence}
\end{lemma}

\begin{proof}
First, notice that for all $M$, $\left\|\P_{\T}(M)\right\|_2\leq 2\left\|M\right\|_2$ and hence,
\begin{equation}
\begin{aligned}
\left\|Z_L\right\|_2&= \left\|\P_{\T}(\lambda_L\tU\tV^{T}) - W_0\right\|_2\qquad\text{(Construction of $W_2$)}\\
&\leq \left\|\P_{\T}(\lambda_L\tU\tV^{T})\right\|_2 + \left\|W_0\right\|_2 \qquad\text{(Triangle Inequality)}\\
&\leq 2\left\|\lambda_L\tU\tV^{T}\right\|_2 + \left\|W_0\right\|_2\qquad\text{(Projection Properties)} \\&\leq (3-\alpha)\lambda_L.\\
\end{aligned}
\nonumber
\end{equation}

Using this, we can bound both $W_1$ and $W_2$. For $W_2$, we have
\begin{equation}
\begin{aligned}
\left\|W_2\right\|_2&\leq \left\|\P_{\T^c}(\lambda_L\tU\tV^{T})\right\|_2 + \left\|W_0\right\|_2\qquad\text{(Triangle Inequality)}\\
&= \left\|\P_{\T^c}(\lambda_L\tU\tV^{T}-\lambda_LU^*V^{*T}-W_1)\right\|_2 + \left\|W_0\right\|_2\qquad\text{(Null Space of $\T^c$)}\\
&\leq \left\|\lambda_L\tU\tV^{T}-\lambda_LU^*V^{*T}-W_1\right\|_2 + \left\|W_0\right\|_2\qquad\text{(Projection Properties)}\\
&= \left\|\P_{\tT}(Z_L)-\P_\T(Z_L)\right\|_2 + \left\|W_0\right\|_2\qquad\text{(Construction)}\\
&\leq \rho_0\left\|Z_L\right\|_2 + \left\|W_0\right\|_2\qquad\text{(Oracle Constraint)}\\
&\leq \left((3-\alpha)\rho_0 + (1-\alpha)\right)\lambda_L < \lambda_L.\\
\end{aligned}
\nonumber
\end{equation}

Since $\P_{\tT}(Z_L)=\lambda_L\tU\tV^{T}$, we can establish
\begin{equation}
\begin{aligned}
\left\|W_1\right\|_2&= \left\|\P_\T(\lambda_L\tU\tV^{T})-\lambda_LU^*V^{*T}\right\|_2\\
&\leq \left\|\P_\T(Z_L)-\P_{\tT}(Z_L)\right\|_2 +\left\|\lambda_L\tU\tV^{T}-\lambda_LU^*V^{*T}\right\|_2\qquad\text{(Triangle inequality)}\\
&\leq \rho_0\left\|Z_L\right\|_2 + \rho_0\lambda_L \qquad\text{(Oracle Constraint)}\\ &\leq \left((3-\alpha)\rho_0 + \rho\right)\lambda_L.
\end{aligned}
\nonumber
\end{equation}

This concludes the proof of the lemma.\\
\end{proof}

\section{Proof of Lemma~\ref{lem:tildecond}}
\label{sec:lemtilde}

\noindent \textbf{General Notation:} For a matrix $X\in\mathbb{R}^{a\times b}$, we use  $X^{(1)},\ldots,X^{(a)}$ to denote rows, $X_1,\ldots,X_b$ to denote columns and $X_1^{(1)},\ldots,X_b^{(a)}$ to denote entries. Also, for the sets of indecies $\mathcal{S}_1\subseteq\{1,\cdot\cdot\cdot,a\}$ and $\mathcal{S}_2\subseteq\{1,\cdot\cdot\cdot,b\}$, the matrix $X_{\mathcal{S}_1\mathcal{S}_2}\in\real^{\left|\mathcal{S}_1\right|\times\left|\mathcal{S}_2\right|}$ represents the sub-matrix of $X$ consisting of the rows and columns corresponding to index sets $\mathcal{S}_1$ and $\mathcal{S}_2$.\\

We prove part (b) of the lemma. By triangle inequality, we have

\small\begin{equation}
\begin{aligned}
\left\|\tL-L^*\right\|_2&\leq \left\|\P_{\tT}(\tL-L^*)\!- \P_{\T}(\tL-L^*)\right\|_2 + \left\|\P_{\tT}(L^*)\! - \P_{\T}(L^*)\right\|_2\!+ \left\|\P_{\T}(\tL) - L^*\right\|_2\\ &\leq \rho_0\left\|\tL-L^*\right\|_2 + \rho_0\left\|L^*\right\|_2 + \left\|\P_{\T}(\tL) - L^*\right\|_2\qquad\text{(Oracle Constraint)}\\ &\leq \rho_0\left\|\tL-L^*\right\|_2 + \rho_0\left\|L^*\right\|_2 + \left\|\P_\T(\tU\tV^{T})-U^*V^{*T}\right\|_2\left\|\tL - L^*\right\|_2\qquad\text{(SVD)} \\ &\leq \rho_0\left\|\tL-L^*\right\|_2 + \rho_0\left\|L^*\right\|_2 + 4\rho_0\left\|\tL - L^*\right\|_2\qquad\text{by {\bf (C3)}}.
\end{aligned}
\nonumber
\end{equation}\normalsize

Hence,
\small\begin{equation}
\begin{aligned}
\left\|\tL-L^*\right\|_\infty \leq \left\|\tL-L^*\right\|_2 \leq \frac{\rho_0}{1-5\rho_0}\left\|L^*\right\|_2.
\end{aligned}
\label{eq:L-norm}
\end{equation}\normalsize\\

\bigskip

Let $Q^{(n)}=\frac{1}{n}\sum_{i=1}^nx(i)x(i)^T$ and $R^{(n)}=\frac{1}{n}\sum_{i=1}^nu(i)x(i)^T$. Substituting $x(i+1)-x(i)=\eta A^*x(i) + \eta B^*u(i) +w(i)$ and $L^*=B^*R^*(Q^*)^{-1}$ in {\bf (C5)}, we get

\small\begin{equation}
\begin{aligned}
(\tA\!-A^*)Q^{(n)}\! &+ (\tL-L^*)Q^{(n)}\!-  \underbrace{B^*\!\left(R^{(n)}-R^*(Q^*)^{-1}Q^{(n)}\right)}_{Y^{(n)}}- \underbrace{\frac{1}{n\eta}\sum_{i=1}^n w(i)x(i)^T}_{W^{(n)}} + \EstDual\!\! =\! 0.
\end{aligned}
\label{eq:err-bnd}
\end{equation}\normalsize

We can rewrite this equation as
\begin{equation}
\begin{aligned}
\P_{\O^c}(\tL-L^*)Q^{(n)} &+ (\tA-A^*+\P_\O(\tL-L^*))Q^{(n)}-  Y^{(n)} - W^{(n)} + \EstDual = 0.
\end{aligned}
\label{eq:err-bnd-clean}
\end{equation}

Let us only focus on the $\svert^{th}$ row of the system of equations \eqref{eq:err-bnd}. We can break down \eqref{eq:err-bnd} on the $\svert^{th}$ row into two sets of linear equations as follows:
\begin{equation}
\begin{aligned}
&(\tA-A^*+\tL-L^*)^{(\svert)}_{\S_\svert}Q^{(n)}_{\S_\svert\S_\svert}= -  (\tL-L^*)^{(\svert)}_{\S_\svert^c}Q^{(n)}_{\S_\svert^c\S_\svert} + Y^{(n)}_{\S_\svert} + W^{(n)}_{\S_\svert} - \EstDual_{\S_\svert}\\
&(\tA-A^*+\tL-L^*)^{(\svert)}_{\S_\svert}Q^{(n)}_{\S_\svert\S_\svert^c}= -  (\tL-L^*)^{(\svert)}_{\S_\svert^c}Q^{(n)}_{\S_\svert^c\S_\svert^c} + Y^{(n)}_{\S_\svert^c} + W^{(n)}_{\S_\svert^c} - \EstDual_{\S_\svert^c}.
\end{aligned}
\label{eq:err-bnd-row}
\end{equation}

From the first line, we get

\small\begin{equation}
\begin{aligned}
\tA-A^*= L^*-\tL -  \left((\tL-L^*)^{(\svert)}_{\S_\svert^c}Q^{(n)}_{\S_\svert^c\S_\svert} + Y^{(n)}_{\S_\svert} + W^{(n)}_{\S_\svert} - \EstDual_{\S_\svert}\right)\left(Q^{(n)}_{\S_\svert\S_\svert}\right)^{-1}\\
\end{aligned}
\nonumber
\end{equation}\normalsize

By Lemma~\ref{lem:Q-Incoherence-bound}, we have

\small\begin{equation}
\begin{aligned}
\left\|(\tL-L^*)^{(\svert)}_{\S_\svert^c}Q^{(n)}_{\S_\svert^c\S_\svert}\left(Q^{(n)}_{\S_\svert\S_\svert}\right)^{-1}\right\|_\infty \leq \left(1-\frac{\theta}{2}\right) \left\|\P_{\O^c}(\tL-L^*)\right\|_\infty\leq \left\|\tL-L^*\right\|_\infty.
\end{aligned}
\nonumber
\end{equation}\normalsize

Thus, by Lemma~\ref{lem:Q-Cmin-bound} and {\bf (C1)}, we get
\begin{equation}
\begin{aligned}
\left\|\tA-A^*\right\|_\infty&\leq \frac{2\rho_0}{1-5\rho_0}\left\|L^*\right\|_2 + \frac{\sqrt{s}}{\mathcal{C}_{\min}}\left(\left\|Y^{(n)}\right\|_\infty\!\!\!\! +\left\|W^{(n)}\right\|_\infty\!\!\!\! + \lambda_{A}\right)\\ &\leq \frac{2\rho_0}{1-5\rho_0}\left\|L^*\right\|_2 + \frac{(8-\theta)\lambda_A\sqrt{s}}{\mathcal{C}_{\min}(4-\theta)}\qquad\text{(Lemmas~\ref{lem:W-bound},\ref{lem:Y-bound})}\\
&\leq \left(\frac{\alpha\theta}{\mathcal{D}_{\max}\left(1+\frac{\mathcal{D}_{\max}}{\mathcal{C}_{\min}}\right)} +\frac{(8-\theta)\sqrt{s}}{\mathcal{C}_{\min}(4-\theta)}\right)\lambda_A=\nu\lambda_A.
\end{aligned}
\label{eq:A-norm}
\end{equation}
The last inequality follows from the definition of $\rho_0$. This concludes the proof of the lemma.

\begin{lemma}[Convex Optimality]\label{LemOptNoisy}
If $\hA$ is a solution of \eqref{eq:opt-orig} then there exists a matrix $\widehat{Z}\in\real^{p\times p}$, called \emph{dual variable}, such that $\widehat{Z}\in \lambda_A \partial \|\hA\|_{1}$ and $\widehat{Z}\in \lambda_L \partial \|\hL\|_{*}$ and 
\begin{equation}\label{EqnConvOptCond}
-\frac{1}{\eta n} \sum_{i=1}^n \left(x(i+1)-x(i)-\eta(\hA+\hL)x(i)\right)x(i)^T + \widehat{Z} = 0.
\end{equation}
\end{lemma}

\begin{proof}
The proof follows from the standard first order optimality argument.\\
\end{proof}

\section{Proof of Lemma~\ref{lem:TtTConnection}}

The result follows from our construction of $W_0$, $W_1$ and $W_2$ in the proof of Lemma~\ref{LemWitnessOptCond}. With our dual construction, we have $\P_{\tT}(J_n)=\lambda_L\tU\tV^T$ and hence, $J_n = \lambda_L\tU\tV^T + W_2$ and by construction, $J_n = \P_\T(\lambda_L\tU\tV^T) + W_0$ which entails $\P_\T(J_n)=\P_\T(\lambda_L\tU\tV^T)$. This concludes the proof of the lemma.\\

\section{Proof of Lemma~\ref{lem:dual-feasibility}}

Substituting $(\tA-A^*+\tL-L^*)^{(\svert)}_{\S_\svert}$ from the first equation in the second in \eqref{eq:err-bnd-row}, we get
\begin{equation}
\begin{aligned}
&\EstDual_{\S_\svert^c} = -  (\tL-L^*)^{(\svert)}_{\S_\svert^c}Q^{(n)}_{\S_\svert^c\S_\svert^c} + Y^{(n)}_{\S_\svert^c} + W^{(n)}_{\S_\svert^c}\\
&\qquad\qquad\qquad-\!\left(\!- (\tL-L^*)^{(\svert)}_{\S_\svert^c}Q^{(n)}_{\S_\svert^c\S_\svert}\!\!\! + Y^{(n)}_{\S_\svert}\!\! + W^{(n)}_{\S_\svert}\!\! - \EstDual_{\S_\svert}\right)\!\! \left(\!Q^{(n)}_{\S_\svert\S_\svert}\!\right)^{\!\!-1}\!\!\!\!\! Q^{(n)}_{\S_\svert\S_\svert^c}.
\end{aligned}
\nonumber
\end{equation}

By triangle inequality, we get

\small
\begin{equation}
\begin{aligned}
\left\|\P_{\O^c}(\EstDual)\right\|_\infty &\leq   \max_\svert\left\|(\tL-L^*)^{(\svert)}_{\S_\svert^c} \left(Q^{(n)}_{\S_\svert^c\S_\svert^c}- Q^{(n)}_{\S_\svert^c\S_\svert}\left(Q^{(n)}_{\S_\svert\S_\svert}\right)^{\!-1}\!\!\!\! Q^{(n)}_{\S_\svert\S_\svert^c}\right)\right\|_\infty\\ &\qquad+ \left\|Y^{(n)}\right\|_\infty + \left\|W^{(n)}\right\|_\infty \\ & \qquad +\max_\svert\left\|Q^{(n)}_{\S_\svert^c\S_\svert}\left(Q^{(n)}_{\S_\svert\S_\svert}\right)^{-1}\right\|_{\infty,1}\!\! \left(\left\|Y^{(n)}\right\|_\infty\!\!\!\! +\left\|W^{(n)}\right\|_\infty\!\!\!\! + \lambda_{A}\right)\\
&\leq \frac{2\rho_0}{1-5\rho_0}\left(1+\frac{\mathcal{D}_{\max}}{\mathcal{C}_{\min}}\right) \mathcal{D}_{\max}\left\|L^*\right\|_2 \qquad\text{(Lemma~\ref{lem:schur})}\\ 
&\qquad+ \frac{\theta\lambda_A}{4(4-\theta)} + \frac{\theta\lambda_A}{4(4-\theta)} \qquad\text{(Lemmas~\ref{lem:W-bound},\ref{lem:Y-bound})}\\ 
& \qquad +\left(1-\frac{\theta}{2}\right)\left(\frac{\theta\lambda_A}{4(4-\theta)} +\frac{\theta\lambda_A}{4(4-\theta)} + \lambda_A\right)\qquad\text{(Lemmas~\ref{lem:Q-Incoherence-bound},\ref{lem:W-bound},\ref{lem:Y-bound})}\\ 
&= \frac{2\rho_0}{1-5\rho_0}\left(1+\frac{\mathcal{D}_{\max}}{\mathcal{C}_{\min}}\right) \mathcal{D}_{\max}\left\|L^*\right\|_2 + \left(1 - \frac{\theta}{4}\right)\lambda_A\\ 
&\leq \left(1 - \frac{(1-\alpha)\theta}{4}\right)\lambda_A.
\end{aligned}
\nonumber
\end{equation}\normalsize
Hence, condition {\bf (C2)} is satisfied.\\

To show {\bf (C4)} also holds, notice that from \eqref{eq:err-bnd-clean}, we have
\begin{equation}
\begin{aligned}
\left\|\P_{\T^c}(\EstDual)\right\|_2&\quad\leq \left\|\P_{\T^c}\left((\tA\!+\tL\!-A^*\!-L^*)Q^{(n)}\right)\right\|_2\!\! + \left\|Y^{(n)}\right\|_2\!\! + \left\|W^{(n)}\right\|_2\\ &\quad\leq  \left\|\P_{\T^c}\left((\tA\!+\tL\!-A^*\!-L^*)Q^{(n)}\right)\right\|_2 + \frac{\theta\lambda_A\sqrt{p}}{2(4-\theta)}.
\end{aligned}
\nonumber
\end{equation}
The last inequality follows from Lemmas~\ref{lem:W-bound} and \ref{lem:Y-bound} and the fact that $Q^{(n)}$ on the support is invertible for the given sample complexity due to Lemma~\ref{lem:Q-Cmin-bound}.

Next, notice that $L^*=B^*R^*(Q^*)^{-1}$ and hence the row-space of $L^*$ is the column/row space of $Q^*$ and consequently, for any matrix $F\in\T$, we have $\P_{\T^c}(FQ^*)=0$. Thus, by triangle inequality, we have

\small\begin{equation}
\begin{aligned}
&\left\|\P_{\T^c}\left((\tA+\tL-A^*-L^*)Q^{(n)}\right)\right\|_2\\ &= \left\|\P_{\T^c}\left((\tA+\tL-A^*-L^*)\left(Q^{(n)}-Q^*\right)\right)\right\|_2 + \left\|\P_{\T^c}\left(\tA+\tL-A^*-L^*\right)Q^*\right\|_2\\
&\leq \left\|(\tA+\tL-A^*-L^*)\left(Q^{(n)}-Q^*\right)\right\|_2+ \left\|\tA+\tL-A^*-L^*\right\|_2\left\|Q^*\right\|_2\sqrt{p}\qquad\text{(Projection Properties)}\\
&\leq \left(\sqrt{s}\left\|\tA\!-A^*\right\|_\infty\!\!\!\!\!+\left\|\tL-L^*\right\|_2\right)\left(\sqrt{p} \left\|Q^{(n)}-Q^*\right\|_\infty+\mathcal{D}_{\max}\right)\sqrt{p}\qquad\text{(Triangle Inequality)}.\\
\end{aligned}
\nonumber
\end{equation}\normalsize

Finally, from \eqref{eq:A-norm}, \eqref{eq:L-norm} and Lemma~\ref{lem:Q-lambda-A}, we get
\small\begin{equation}
\begin{aligned}
\left\|\P_{\T^c}(\EstDual)\right\|_2&\leq \left(\frac{\theta\mathcal{C}_{\min}\sqrt{p}}{9s\sqrt{s}} +\mathcal{D}_{\max}\right)\left(\frac{3\alpha\theta\sqrt{s}} {2\mathcal{D}_{\max}\left(1+\frac{\mathcal{D}_{\max}}{\mathcal{C}_{\min}}\right)} +\frac{(8-\theta)s}{\mathcal{C}_{\min}(4-\theta)}\right)\lambda_A\sqrt{p} +\frac{\theta\lambda_A\sqrt{p}}{2(4-\theta)}\\ &\leq \theta(1-\alpha)\lambda_L \qquad\text{(By {\bf (A4-2)})}.
\end{aligned}
\nonumber
\end{equation}\normalsize

Hence, condition {\bf (C4)} is also satisfied. This concludes the proof of the lemma.

\section{Concentration Results}
\label{sec:concent}

\noindent In this section we prove the concentration results used throughout the paper. Before, we state the results, we want to introduce some useful notations and inequalities used to get the results. By the dynamics of the system, we have
\begin{equation}
X(i)=\left[\begin{array}{c} x(i) \\ u(i) \end{array}\right] = \left(I+\eta\mathcal{A}^*\right)^i \underbrace{\left[\begin{array}{c} x(0) \\ u(0) \end{array}\right]}_{X(0)} + \sum_{l=0}^{i-1}\left(I+\eta\mathcal{A}^*\right)^{i-l-1} w(l).
\nonumber
\end{equation}

\begin{lemma}
Under assumptions {\bf (A3)} and {\bf (A5)}, for any $\S\subseteq\{1,2,\ldots,p\}$ with $|\S|\leq s$, with high probability we have
\begin{equation}
\left\|\mathcal{Q}^{(n)}_{\S^c\S}\left(\mathcal{Q}^{(n)}_{\S\S}\right)^{-1}\right\|_{\infty,1}\leq 1 - \frac{\theta}{2}.
\nonumber
\end{equation}
\label{lem:Q-Incoherence-bound}
\end{lemma}

\begin{proof}
Using Lemma~\ref{lem:Q-Cmin-bound}, it can be shown (see Lemma A.1 in \cite{BENIBRMON} for example) that

\small\begin{equation}
\begin{aligned}
\left\|\mathcal{Q}^{(n)}_{\S^c\S}\left(\mathcal{Q}^{(n)}_{\S\S}\right)^{-1}\right\|_{\infty,1}&\leq \left\|\mathcal{Q}^*_{\S^c\S}\left(\mathcal{Q}^*_{\S\S}\right)^{-1}\right\|_{\infty,1}\\ &\qquad+\frac{3|\S|\sqrt{|\S|}}{\mathcal{C}_{\min}}\left\|\mathcal{Q}^{(n)}-\mathcal{Q}^*\right\|_{\infty} +\frac{2|\S|^2\sqrt{|\S|}}{\mathcal{C}_{\min}^2}\left\|\mathcal{Q}^{(n)}-\mathcal{Q}^*\right\|_{\infty}^2.
\end{aligned}
\nonumber
\end{equation}\normalsize
The result follows from Lemma~\ref{lem:Q-bound}. This concludes the proof of the lemma.\\
\end{proof}

\begin{lemma}
Under assumption {\bf (A5)}, for any $\S\subseteq\{1,2,\ldots,p\}$ with $|\S|\leq s$, with high probability, we have
\begin{equation}
\Lambda_{\min}\left(\mathcal{Q}^{(n)}_{\S\S}\right)\geq \frac{\mathcal{C}_{\min}}{2}.
\nonumber
\end{equation}
\label{lem:Q-Cmin-bound}
\end{lemma}

\begin{proof}
By the Courant-Fischer variational representation \cite{HORJOH}, we have
\begin{equation}
\begin{aligned}
\Lambda_{\min}\left(\mathcal{Q}^{(n)}_{\S\S}\right) &\geq \Lambda_{\min}\left(\mathcal{Q}^*_{\S\S}\right) - \Lambda_{\max}\left(\mathcal{Q}^*_{\S\S} - \mathcal{Q}^{(n)}_{\S\S}\right)\\
&\geq \mathcal{C}_{\min} - \sqrt{s}\left\|\mathcal{Q}^* - \mathcal{Q}^{(n)}\right\|_\infty.
\end{aligned}
\nonumber
\end{equation}
The last inequality follows from Lemma~\ref{lem:Q-bound}. This concludes the proof of the lemma.\\
\end{proof}

\begin{lemma}
Under assumptions {\bf (A4)} and {\bf (A5)}, with high probability, we have
\begin{equation}
\left\|\mathcal{W}^{(n)}\right\|_\infty\leq\frac{\theta\lambda_A}{4(4-\theta)}.
\nonumber
\end{equation}
\label{lem:W-bound}
\end{lemma}

\begin{proof}
Let $X(i) = \left[ x(i)\,\, u(i)\right]^T$. According to the dynamics of the system, we have
\begin{equation}
\begin{aligned}
\mathcal{W}^{(n)} &= \frac{1}{\eta n}\sum_{i=0}^{n-1}w(i)\underbrace{X(0)^T\left((I+\eta \mathcal{A}^*)^i\right)^T}_{E_1(i)}\\ &\qquad+\frac{1}{\eta n}\sum_{i=1}^{n-1}w(i) \underbrace{\sum_{l=0}^{i-1}w(l)^T\left((I+\eta \mathcal{A}^*)^{i-l-1}\right)^T}_{E_2(i)}.
\end{aligned}
\nonumber
\end{equation}
We bound these two terms separately. Notice that $w(i)$ is distributed $\N(0,\eta I)$ independent of $x(0)$ and $w(j)$'s. Given $x(0)$, we have
\begin{equation}
w(i)_jE_1(i)^{(k)} \sim\N\left(0,\eta\left(E_1(i)^{(k)}\right)^2\right).
\nonumber
\end{equation}
By stability assumption, we have $\left(E_1(i)^{(k)}\right)^2\leq(\|x(0)\|_2^2+\|u(0)\|_2^2)\Sigma_{\max}^{2i}$ and hence,
\begin{equation}
\begin{aligned}
\text{VAR}\left(\frac{1}{\eta n}\sum_{i=0}^{n-1}w(i)_jE_1(i)^{(k)}\right) &\leq\frac{1}{\eta^2n^2}\sum_{i=0}^{n-1}\text{VAR}\left(w(i)_jE_1(i)^{(k)}\right)\\
&\leq\frac{\|x(0)\|_2^2+\|u(0)\|_2^2}{\eta n(1-\Sigma_{\max}^2)}.
\end{aligned}
\nonumber
\end{equation}
Consequently, by standard concentration of Gaussian random variables and union bound, we get
\begin{equation}
\begin{aligned}
\mathbb{P}\left[\left\|\frac{1}{\eta n}\sum_{i=0}^{n-1}w(i)E_1(i)\right\|_\infty\geq\epsilon\right] &\leq\sum_{j=1}^p\sum_{k=1}^p\mathbb{P}\left[\left|\frac{1}{\eta n}\sum_{i=0}^{n-1}w(i)_jE_1(i)^{(k)}\right|\geq\epsilon\right]\\
&\leq 2\exp\Big(-\frac{\epsilon^2(1-\Sigma_{\max}^2)}{2\left(\|x(0)\|_2^2+\|u(0)\|_2^2\right)}\eta n+\log((s+2r)p+r^2)\Big).
\end{aligned}
\nonumber
\end{equation}

With similar analysis, we get
\begin{equation}
\begin{aligned}
\text{VAR}\left(\frac{1}{\eta n}\sum_{i=0}^{n-1}w(i)_jE_2(i)^{(k)}\right)&\leq\frac{1}{\eta^2n^2}\sum_{i=0}^{n-1}\text{VAR}\left(w(i)_jE_2(i)^{(k)}\right)\\
&\leq\frac{\left(\sqrt{\eta}+1\right)^2}{\eta n(1-\Sigma_{\max}^2)}.
\end{aligned}
\nonumber
\end{equation}
The last inequality follows from the concentration of $\chi^2$ random variables \cite{LAUMAS}, in particular,
\begin{equation}
\begin{aligned}
\mathbb{P}\left[\frac{1}{\eta n}\sum_{l=0}^{n-2}\|w(l)_jE_2(l)^{(k)}\|_2^2\geq\frac{\left(1+\sqrt{\eta}\right)^2}{1-\Sigma_{\max}^2}\right] \leq \exp\left(-\frac{1}{2}\eta n+\log((s+2r)p+r^2)\right).
\end{aligned}
\nonumber
\end{equation}
Finally, we get
\begin{equation}
\begin{aligned}
\mathbb{P}\left[\left\|\frac{1}{\eta n}\sum_{i=0}^{n-1}w(i)E_2(i)\right\|_\infty\geq\epsilon\right]&\leq\sum_{j=1}^p\sum_{k=1}^p\mathbb{P}\left[\left|\frac{1}{\eta n}\sum_{i=0}^{n-1}w(i)_jE_2(i)^{(k)}\right|\geq\epsilon\right]\\
&\leq 2\exp\left(-\frac{\epsilon^2(1-\Sigma_{\max}^2)}{2\left(\sqrt{\eta}+1\right)^2}\eta n+\log((s+2r)p+r^2)\right).
\end{aligned}
\nonumber
\end{equation}

The result follows for $\epsilon=\frac{\theta\lambda_A}{8(4-\theta)}$. This concludes the proof of the lemma.\\
\end{proof}

\begin{lemma}
Under assumptions {\bf (A4)} and {\bf (A5)}, with high probability, we have
\begin{equation}
\left\|Y^{(n)}\right\|_\infty\leq \frac{\theta\lambda_A}{4(4-\theta)}.
\nonumber
\end{equation}
\label{lem:Y-bound}
\end{lemma}

\begin{proof}
We can establish
\begin{equation}
Y^{(n)} = \underbrace{B^*\left(R^{(n)} - R^*\right)} + \underbrace{B^*R^*(Q^{*})^{-1}\left(Q^* - Q^{(n)}\right)}.
\nonumber
\end{equation}
We bound these two terms separately. For the first term, we have
\begin{equation}
\left\|B^*\left(R^*-R^{(n)}\right)\right\|_\infty \leq \left\|B^*\right\|_{\infty,1}\left\|\mathcal{Q}^* - \mathcal{Q}^{(n)}\right\|_\infty.\\
\nonumber
\end{equation}

For the second term, we have 
\begin{equation}
\begin{aligned}
\left\|B^*R^*(Q^*)^{-1}\left(Q^*-Q^{(n)}\right)\right\|_\infty&\leq \left\|B^*R^*(Q^*)^{-1}\right\|_{\infty,1}\left\|Q^*-Q^{(n)}\right\|_\infty\\ &\leq \left\|B^*\right\|_{\infty,1}\sigma_{\max}\left(R^*(Q^*)^{-1}\right)\left\|Q^*-Q^{(n)}\right\|_\infty\\
&\leq \left\|B^*\right\|_{\infty,1}\frac{\mathcal{D}_{\max}}{\mathcal{C}_{\min}}\left\|Q^*-Q^{(n)}\right\|_\infty.
\end{aligned}
\nonumber
\end{equation}

The result follows from Lemma~\ref{lem:Q-lambda-A}. This concludes the proof of the lemma.\\
\end{proof}

\begin{lemma}
Under assumption {\bf (A5)}, with high probability, we have
\begin{equation}
\left\|\mathcal{Q}^* - \mathcal{Q}^{(n)}\right\|_\infty\leq \frac{\theta\,\mathcal{C}_{\min}}{9\,s\sqrt{s}}.
\nonumber
\end{equation}
\label{lem:Q-bound}
\end{lemma}

\begin{proof}
Let $X(i) = \left[ x(i)\,\, u(i)\right]^T$. Let $\mu(i) = \mathbb{E}\left[X(i)\right]$ (clearly, $\mu(\infty)=0$). We have 
\small\begin{equation}
\begin{aligned}
\mathcal{Q}^{(n)} - \mathcal{Q}^* &= \underbrace{\frac{1}{n}\sum_{i=0}^{n-1}\mu(i)\mu(i)^T} + \underbrace{\frac{1}{n}\sum_{i=0}^{n-1}\mathbb{E}\left[\left(X(i)-\mu(i)\right)\left(X(i)-\mu(i)\right)^T\right] - \mathcal{Q}^*}_{E_1}\\ &\qquad+ \underbrace{\frac{1}{n}\sum_{i=0}^{n-1}\left(X(i)-\mu(i)\right)\left(X(i)-\mu(i)\right)^T-(E_1+\mathcal{Q}^*)}_{E_2}.
\end{aligned}
\nonumber
\end{equation}\normalsize
We bound these three terms, separately. For the first term, we have
\begin{equation}
\begin{aligned}
\left\|\frac{1}{n}\sum_{i=0}^{n-1}\mu(i)\mu(i)^T\right\|_\infty &\leq \frac{1}{n}\sum_{i=0}^{n-1}\Sigma_{\max}^{2i}\left(\left\|x(0)\right\|_2^2+\left\|u(0)\right\|_2^2\right)\\
&\leq\frac{\left\|x(0)\right\|_2^2+\left\|u(0)\right\|_2^2}{n(1-\Sigma_{\max}^2)}.
\end{aligned}
\nonumber
\end{equation}

For the second term, notice that by independency assumption on $w$, we have
\begin{equation}
\begin{aligned}
&\frac{1}{n}\sum_{i=0}^{n-1}\mathbb{E}\left[\left(X(i)-\mu(i)\right)\left(X(i)-\mu(i)\right)^T\right]\\ &\qquad\qquad= \frac{\eta}{n}\sum_{i=0}^{n-1}\sum_{l=0}^{i-1}\left(I+\eta\mathcal{A}^*\right)^{2l}\\
&\qquad\qquad= \frac{\eta}{n}\sum_{i=0}^{n-1}\left(I-\left(I+\eta\mathcal{A}^*\right)^{2i}\right) \left(I-\left(I+\eta\mathcal{A}^*\right)^2\right)^{-1}\\
&\qquad\qquad= \eta\left(\frac{n-1}{n}I - \left(I+\eta\mathcal{A}^*\right)^{2} +\frac{1}{n}\left(I+\eta\mathcal{A}^*\right)^{2n}\right)\left(I-\left(I+\eta\mathcal{A}^*\right)^2\right)^{-2}.
\end{aligned}
\nonumber
\end{equation}
On the other hand, we have
\begin{equation}
\begin{aligned}
\mathcal{Q}^* &= \mathbb{E}\left[\lim_{i\rightarrow\infty}\left(X(i)-\mu(i)\right)\left(X(i)-\mu(i)\right)^T\right] \\ &= \lim_{i\rightarrow\infty}\mathbb{E}\left[\left(X(i)-\mu(i)\right)\left(X(i)-\mu(i)\right)^T\right]\\
&= \lim_{i\rightarrow\infty}\eta\sum_{l=0}^{i-1}\left(I+\eta\mathcal{A}^*\right)^{2l}\\
&= \lim_{i\rightarrow\infty}\eta\left(I-\left(I+\eta\mathcal{A}^*\right)^{2i}\right) \left(I-\left(I+\eta\mathcal{A}^*\right)^2\right)^{-1}\\
&= \eta\left(I-\left(I+\eta\mathcal{A}^*\right)^2\right)^{-1}.
\end{aligned}
\nonumber
\end{equation}
In the above inequalities, we interchanged limit and expectation as a result of Gaussianity assumption and the stability of the system. Finally we get
\begin{equation}
\left\|E_1\right\|_\infty \leq \frac{\eta(1-\Sigma_{\max}^{2n})}{n(1-\Sigma_{\max}^2)^2}\leq \frac{\eta}{n(1-\Sigma_{\max}^2)^2}.
\nonumber
\end{equation}

To bound the third term, notice that 
\footnotesize\begin{equation}
\begin{aligned}
\frac{1}{n}\sum_{i=0}^{n-1}\left(X(i)-\mu(i)\right)\left(X(i)-\mu(i)\right)^T = \sum_{j=0}^{n-1}(I+\eta \mathcal{A}^*)^j\!\!\left(\frac{n-j}{n}\underbrace{\frac{1}{n-j}\sum_{i=0}^{n-j-1}w(i)w(i)^T\!\!}_{V_j}\right) \!\!\left((I+\eta \mathcal{A}^*)^j\right)^{\!\!T}.
\end{aligned}
\nonumber
\end{equation}\normalsize
By Lemma 1 in \cite{RAVWAIRASYU}, we have
\begin{equation}
\begin{aligned}
\mathbb{P}\left[\left\|V_j-\eta I\right\|_\infty>\frac{n}{n-j}\epsilon\right] \leq 4\exp\left(-\frac{\epsilon^2n}{3200\eta(n-j)}n+\log\left((s+2r)p+r^2\right)\right).
\end{aligned}
\nonumber
\end{equation}
Consequently, we get
\begin{equation}
\begin{aligned}
&\mathbb{P}\left[\frac{n-j}{n}\Sigma_{\max}^{2(n-j-1)}\left\|V_j-\eta I\right\|_\infty>\Sigma_{\max}^{2(n-j-1)}\epsilon\right]\\&\qquad\qquad\qquad\qquad\leq 4\exp\left(-\frac{\epsilon^2n}{3200\eta(n-j)}n+\log\left((s+2r)p+r^2\right)\right).
\end{aligned}
\nonumber
\end{equation}
Thus, we conclude
\begin{equation}
\begin{aligned}
&\mathbb{P}\left[\left\|E_2\right\|_\infty>\frac{1}{1-\Sigma_{\max}^2}\epsilon\right]\leq 4\exp\left(-\frac{\epsilon^2}{3200\eta}n+\log\left((s+2r)p+r^2\right)\right).
\end{aligned}
\nonumber
\end{equation}
We want this probability to be less than $\delta$. Putting all thre parts together, we get
\begin{equation}
\begin{aligned}
&\left\|\mathcal{Q}^* - \mathcal{Q}^{(n)}\right\|_\infty\leq \frac{1}{1-\Sigma_{\max}^2}\left(\frac{\eta(1-\Sigma_{\max}^2)^{-1}+\left\|x(0)\right\|_2^2+\left\|u(0)\right\|_2^2}{n} +\epsilon\right).
\end{aligned}
\label{eq:Q-fundamental-bound}
\end{equation}
For $n\eta\geq \frac{18\, s\sqrt{s}}{D\, \theta\, \mathcal{C}_{\min}}\left(D^{-1} +\left\|x(0)\right\|_2^2+\left\|u(0)\right\|_2^2\right)$ and $\epsilon = \frac{\eta D \theta \mathcal{C}_{\min}}{18\,s\sqrt{s}}$, the result follows, provide that the probabilities go to zero, i.e., $$n\eta\geq\frac{3\times 10^6\,s^3}{D^2\,\theta^2\,\mathcal{C}_{\min}^2}\log\left(\frac{4((s+2r)p+r^2)}{\delta}\right).$$ For large enough values of $p$, this lower bound dominates the earlier lower bound of $n\eta$, hence, we ignore that one. This concludes the proof of the lemma.\\
\end{proof}

\begin{lemma}
Under assumptions {\bf (A4)} and {\bf (A5)}, with high probability, we have
\begin{equation}
\left\|\mathcal{Q}^* - \mathcal{Q}^{(n)}\right\|_\infty\leq \frac{\theta\lambda_A}{4(4-\theta)\,\left\|B^*\right\|_{\infty,1}\left(\frac{\mathcal{D}_{\max}}{\mathcal{C}_{\min}}+1\right)}.
\nonumber
\end{equation}
\label{lem:Q-lambda-A}
\end{lemma}

\begin{proof}
According to \eqref{eq:Q-fundamental-bound}, the result follows if $\epsilon = \frac{\theta\lambda_A\,D}{8(4-\theta)\left\|B^*\right\|_{\infty,1}\left(\frac{\mathcal{D}_{\max}}{\mathcal{C}_{\min}}+1\right)}$ assuming $p$ is large enough.\\
\end{proof}

\begin{lemma}
For sample complexity $$n\eta\geq \frac{3\times 10^6 \left(\mathcal{D}_{\max}+2\mathcal{C}_{\min}\right)} {D^2\left(\mathcal{D}_{\max}+\mathcal{C}_{\min}\right)}\log\left(\frac{4((s+2r)p+r^2)}{\delta}\right)$$ with high probability, we have
\begin{equation}
\underbrace{\left\|Q^{(n)}_{\S_\svert^c\S_\svert^c}- Q^{(n)}_{\S_\svert^c\S_\svert}\left(Q^{(n)}_{\S_\svert\S_\svert}\right)^{\!-1}\!\!\!\! Q^{(n)}_{\S_\svert\S_\svert^c}\right\|_2}_{S^{(n)}}\leq 2(1+\frac{\mathcal{D}_{\max}}{\mathcal{C}_{\min}})\mathcal{D}_{\max}.
\nonumber
\end{equation}
\label{lem:schur}
\end{lemma}

\begin{proof}
Since $Q^*$ and $Q^{(n)}$ are positive semi-definite matrices and
\begin{equation}
\left\|S^{(n)}\right\|_2\leq\left\|S^{(n)}-S^*\right\|_2+\left\|S^*\right\|_2\\
\nonumber
\end{equation}
The result directly follows from Theorem in \cite{STE} for $\epsilon\defn\|Q^{(n)}-Q^*\|_\infty=\frac{\mathcal{D}_{\max}+\mathcal{C}_{\min}} {4\left(\mathcal{D}_{\max}+2\mathcal{C}_{\min}\right)}$ considering the fact that $\left\|S^*\right\|_2\leq \mathcal{D}_{\max}\left(1+\frac{\mathcal{D}_{\max}}{\mathcal{C}_{\min}}\right)$.\\
\end{proof}

\bibliographystyle{plainnat}
\bibliography{SystemID}

\begin{thebibliography}{46}
\providecommand{\natexlab}[1]{#1}
\providecommand{\url}[1]{\texttt{#1}}
\expandafter\ifx\csname urlstyle\endcsname\relax
  \providecommand{\doi}[1]{doi: #1}\else
  \providecommand{\doi}{doi: \begingroup \urlstyle{rm}\Url}\fi

\bibitem[Azoff(1994)]{Azo94}
E.M. Azoff.
\newblock \emph{Neural Network Time Series Forecasting of Financial Markets}.
\newblock John Wiley \& Sons, Inc., 1994.

\bibitem[Bar-Joseph(2004)]{BAR}
Z.~Bar-Joseph.
\newblock Analyzing time series gene expression data.
\newblock \emph{Bioinformatics, Oxford University Press}, 20:\penalty0
  2493--2503, 2004.

\bibitem[Bento et~al.(2010)Bento, Ibrahimi, and Montanari]{BENIBRMON}
J.~Bento, M.~Ibrahimi, and A.~Montanari.
\newblock Learning networks of stochastic equations.
\newblock In \emph{NIPS}, 2010.

\bibitem[Bowerman and O'Connell(1993)]{BOWOCO93}
B.L. Bowerman and R.T. O'Connell.
\newblock \emph{Forecasting and time series: An applied approach}.
\newblock Duxbury Press, 1993.

\bibitem[Box et~al.(1990)Box, Jenkins, and Reinsel]{BoxJenRei90}
G.E.P. Box, G.M. Jenkins, and G.C. Reinsel.
\newblock \emph{Time-series Analysis: Forecasting and Control}.
\newblock John Wiley \& Sons, Inc., 1990.

\bibitem[Candes and Plan(2010)]{CANPLA}
E.~J. Candes and Y.~Plan.
\newblock Matrix completion with noise.
\newblock In \emph{IEEE Proceedings}, volume~98, pages 925 -- 936, 2010.

\bibitem[Candes et~al.(2009)Candes, Li, Ma, and Wright]{CANLIMAWRI}
E.~J. Candes, X.~Li, Y.~Ma, and J.~Wright.
\newblock Robust principal component analysis?
\newblock In \emph{Available at arXiv:0912.3599}, 2009.

\bibitem[Chandrasekaran et~al.(2010)Chandrasekaran, Parrilo, and
  Willsky]{CHAPARWIL}
V.~Chandrasekaran, P.~A. Parrilo, and A.~S. Willsky.
\newblock Latent variable graphical model selection via convex optimization.
\newblock In \emph{Available at arXiv:1008.1290}, 2010.

\bibitem[Chandrasekaran et~al.(2011)Chandrasekaran, Sanghavi, Parrilo, and
  Willsky]{CHASANPARWIL}
V.~Chandrasekaran, S.~Sanghavi, P.~A. Parrilo, and A.~S. Willsky.
\newblock Rank-sparsity incoherence for matrix decomposition.
\newblock \emph{SIAM Journal on Optimization}, 2011.

\bibitem[Chatfield(2000)]{CHA00}
C.~Chatfield.
\newblock \emph{Time-series Forecasting}.
\newblock Chapman \& Hall, 2000.

\bibitem[Chen et~al.(2011)Chen, Jalali, Sanghavi, and Caramanis]{CHEJALSANCAR}
Y.~Chen, A.~Jalali, S.~Sanghavi, and C.~Caramanis.
\newblock Low-rank matrix recovery from errors and erasures.
\newblock In \emph{ISIT}, 2011.

\bibitem[Cochrane(2005)]{COC}
J.~H. Cochrane.
\newblock \emph{Time Series for Macroeconomics and Finance}.
\newblock University of Chicago, 2005.

\bibitem[d'Aspremont et~al.(2007)d'Aspremont, Bannerjee, and Ghaoui]{Asp07b}
A.~d'Aspremont, O.~Bannerjee, and L.~El Ghaoui.
\newblock First order methods for sparse covariance selection.
\newblock \emph{{SIAM} Journal on Matrix Analysis and its Applications}, 2007.
\newblock To appear.

\bibitem[Dempster et~al.(1977)Dempster, Laird, and Rubin]{DemLaiRub77}
A.P. Dempster, N.M. Laird, and D.B. Rubin.
\newblock Maximum-likelihood from incomplete datavia the em algorithm.
\newblock \emph{Journal of Royal Statistics Society, Series B.}, 39, 1977.

\bibitem[Fazel et~al.(2011)Fazel, Pong, Sun, and Tseng]{FazPon11}
M.~Fazel, T.K. Pong, D.~Sun, and P.~Tseng.
\newblock Hankel matrix rank minimization with applications in system
  identification and realization.
\newblock In \emph{Available at
  http://faculty.washington.edu/mfazel/Hankelrm9.pdf}, 2011.

\bibitem[Fisher(1925)]{FIS}
R.~A. Fisher.
\newblock Theory of statistical estimation.
\newblock In \emph{Proceedings of Cambridge Philosophy Society}, volume~22,
  pages 700--725, 1925.

\bibitem[Friedman et~al.(2007)Friedman, Hastie, and Tibshirani]{glasso}
J.~Friedman, T.~Hastie, and R.~Tibshirani.
\newblock Sparse inverse covariance estimation with the graphical lasso.
\newblock \emph{BioStatistics}, 9:\penalty0 432--441, 2007.

\bibitem[Ger{\v s}gorin(1931)]{GER}
S.~Ger{\v s}gorin.
\newblock Uber die abgrenzung der eigenwerte einer matrix.
\newblock \emph{Bulletin de l'Acad\'emie des Sciences de l'URSS. Classe des
  sciences math\'ematiques et na}, 7:\penalty0 749--754, 1931.

\bibitem[Gillespie(2007)]{GIL}
D.T. Gillespie.
\newblock Stochastic simulation of chemical kinetics.
\newblock \emph{Annual Review of Physical Chemistry}, 58:\penalty0 35--55,
  2007.

\bibitem[Hazan et~al.(2005)Hazan, Polak, and Shashua]{HazPolsha05}
T.~Hazan, S.~Polak, and A.~Shashua.
\newblock Sparse image coding using a 3d non-negative tensor factorization.
\newblock In \emph{ICCV}, 2005.

\bibitem[Higham(2008)]{HIG}
D.~Higham.
\newblock Modeling and simulating chemical reactions.
\newblock \emph{SIAM Review}, 50:\penalty0 347--368, 2008.

\bibitem[Horn and Johnson(1985)]{HORJOH}
R.~A. Horn and C.~R. Johnson.
\newblock \emph{Matrix Analysis}.
\newblock Cambridge University Press, Cambridge, 1985.

\bibitem[Jalali and Srebro(2012)]{JalSre12}
A.~Jalali and N.~Srebro.
\newblock Clustering using max-norm constrained optimization.
\newblock In \emph{Available at arXiv:1202.5598}, 2012.

\bibitem[Jalali et~al.(2011)Jalali, Chen, Sanghavi, and Xu]{JalCheSanXu11}
A.~Jalali, Y.~Chen, S.~Sanghavi, and H.~Xu.
\newblock Clustering partially observed graphs via convex optimization.
\newblock In \emph{ICML}, 2011.

\bibitem[Jordan(1998)]{JOR}
M.~I. Jordan.
\newblock \emph{Learning in Graphical Models}.
\newblock Kluwer Academic Publishers, Netherland, 1998.

\bibitem[Kim(2003)]{Kim03}
K.~Kim.
\newblock Financial time series forecasting using support vector machines.
\newblock \emph{Elsevier Neurocomputing}, 55:\penalty0 307--319, 2003.

\bibitem[Laurent and Massart(1998)]{LAUMAS}
B.~Laurent and P.~Massart.
\newblock Adaptive estimation of a quadratic functional by model selection.
\newblock \emph{Annals of Statistics}, 28:\penalty0 1303--1338, 1998.

\bibitem[Lawrence et~al.(2010)Lawrence, Girolami, Rattray, and
  Sanguinetti]{LAWGIRRATSAN}
N.~D. Lawrence, M.~Girolami, M.~Rattray, and G.~Sanguinetti.
\newblock \emph{Learning and Inference in Computational Systems Biology}.
\newblock MIT Press, 2010.

\bibitem[Lin et~al.(2009)Lin, Ganesh, Wright, Wu, Chen, and
  Ma]{LINGANWRIWUCHEMA}
Z.~Lin, A.~Ganesh, J.~Wright, L.~Wu, M.~Chen, and Y.~Ma.
\newblock Fast convex optimization algorithms for exact recovery of a corrupted
  low-rank matrix.
\newblock In \emph{UIUC Technical Report UILU-ENG-09-2214}, 2009.

\bibitem[Ljung(1999)]{Lju99}
L.~Ljung.
\newblock \emph{System identification: Theory for the user}.
\newblock Prentice Hall, 1999.

\bibitem[Loehlin(1984)]{Loe04}
J.C. Loehlin.
\newblock \emph{Latent Variable Models: An introduction tofactor, path, and
  structural analysis}.
\newblock L. Erlbaum Associates Inc. Hillsdale, NJ, USA, 1984.

\bibitem[Marchal(2003)]{MAR}
P.~Marchal.
\newblock Constructing a sequence of random walks strongly converging to
  brownian motion.
\newblock In \emph{Discrete Mathematics and Theoretical Computer Science
  Proceedings}, pages 181--190, 2003.

\bibitem[Meinshausen and Buhlmann(2006)]{MEIBUL}
N.~Meinshausen and P.~Buhlmann.
\newblock High-dimensional graphs and variable selection with the lasso.
\newblock \emph{Annals of Statistics}, 34\penalty0 (3):\penalty0 1436--1462,
  2006.

\bibitem[Ravikumar et~al.(2008)Ravikumar, Wainwright, Raskutti, and
  Yu]{RAVWAIRASYU}
P.~Ravikumar, M.~J. Wainwright, G.~Raskutti, and B.~Yu.
\newblock High-dimensional covariance estimation by minimizing
  $\ell_1$-penalized log-determinant divergence.
\newblock \emph{Technical Report 767, UC Berkeley, Department of Statistics},
  2008.

\bibitem[Redner and Walker(1984)]{RenWal84}
R.~Redner and H.~Walker.
\newblock Mixture densities, maximum likelihood and the em algorithm.
\newblock \emph{SIAM Review}, 26, 1984.

\bibitem[Shreve(2004)]{SHR04}
S.~E. Shreve.
\newblock \emph{Stochastic Calculus for Finance II: Continuous-Time Models}.
\newblock Springer, 2004.

\bibitem[Srebro and Jaakkola(2003)]{SreJaa03}
N.~Srebro and T.~Jaakkola.
\newblock Weighted low rank approximation.
\newblock In \emph{ICML}, 2003.

\bibitem[Stewart(1995)]{STE}
G.~W. Stewart.
\newblock On the perturbation of schur complements in positive semidefinite
  matrices.
\newblock \emph{Technical Report, University of Maryland, College Park}, 1995.

\bibitem[Tibshirani(1996)]{TIB}
R.~Tibshirani.
\newblock Regression shrinkage and selection via the lasso.
\newblock \emph{Journal of Royal Statistical Society, Series B}, 58:\penalty0
  267--288, 1996.

\bibitem[Vapnik(1998)]{VAP}
V.~N. Vapnik.
\newblock \emph{Statistical Learning Theory}.
\newblock John Wiley and Sons, Inc., New York, 1998.

\bibitem[Wainwright(2009)]{WAI}
M.~J. Wainwright.
\newblock Sharp thresholds for noisy and high-dimensional recovery of sparsity
  using $\ell_1$-constrained quadratic programming (lasso).
\newblock \emph{IEEE Trans. on Information Theory}, 55:\penalty0 2183--2202,
  2009.

\bibitem[Wei(1994)]{Wei94}
W.W.S. Wei.
\newblock \emph{Time Series Analysis: Univariate and Multivariate Methods}.
\newblock Addison Wesley, 1994.

\bibitem[West(2003)]{West03bayesianfactor}
Mike West.
\newblock Bayesian factor regression models in the "large p, small n" paradigm.
\newblock In \emph{Bayesian Statistics}, pages 723--732. Oxford University
  Press, 2003.

\bibitem[Young(1984)]{YOU}
P.~Young.
\newblock \emph{Recursive estimation and time-series analysis}.
\newblock Springer - Verlag, 1984.

\bibitem[Yuan and Lin(2007)]{YuaLin07}
M.~Yuan and Y.~Lin.
\newblock Model selection and estimation in the {G}aussian graphical model.
\newblock \emph{Biometrika}, 94\penalty0 (1):\penalty0 19--35, 2007.

\bibitem[Zhou et~al.(2010)Zhou, Li, Wright, Candes, and Ma]{ZHOLIWRICANMA}
Z.~Zhou, X.~Li, J.~Wright, E.~Candes, and Y.~Ma.
\newblock Stable principal component pursuit.
\newblock In \emph{ISIT}, 2010.

\end{thebibliography}

\end{document}